\documentclass{article}

\usepackage{arxiv}
\usepackage{amsthm}
\theoremstyle{definition}
\newtheorem{definition}{Definition}[section]
\newtheorem{theorem}{Theorem}[section]
\newtheorem{corollary}{Corollary}[theorem]
\newtheorem{lemma}[theorem]{Lemma}

\usepackage[utf8]{inputenc} 
\usepackage[T1]{fontenc}    
\usepackage{hyperref}       
\usepackage{url}            
\usepackage{booktabs}       
\usepackage{amsfonts}       
\usepackage{nicefrac}       
\usepackage{microtype}      
\usepackage{lipsum}
\usepackage{graphicx}
\usepackage{natbib}

\usepackage{amsmath}
\usepackage{tabularx}
\usepackage{booktabs}
\usepackage{amssymb}
\usepackage{bbm}
\usepackage{hyperref}
\usepackage{cleveref}
\usepackage{graphicx} 
\usepackage{lipsum}
\usepackage{multicol}
\usepackage{scalerel,stackengine}
\usepackage{siunitx}
\usepackage{subfig}

\usepackage{widetext}
\usepackage{sidecap}

\usepackage{arxiv}
\newcommand{\HL}[1]{{\color{red} #1}}
\newcommand{\RV}[1]{{\color{black} #1}}

\newcommand\reallywidehat[1]{\arraycolsep=0pt\relax%
\begin{array}{c}
\stretchto{
  \scaleto{
    \scalerel*[\widthof{\ensuremath{#1}}]{\kern-.5pt\bigwedge\kern-.5pt}
    {\rule[-\textheight/2]{1ex}{\textheight}} 
  }{\textheight} %
}{0.5ex}\\           
#1\\                 
\rule{-1ex}{0ex}
\end{array}
}

\title{Wasserstein projection distance for fairness testing of regression models}

\author{
 Wanxin Li \\
  Department of Computer Science, University of British Columbia\\
  Vancouver, British Columbia, Canada \\
  \texttt{wanxinli@cs.ubc.ca} \\
     \And
 Yongjin P. Park \\
 Department of Pathology and Laboratory Medicine, University of British Columbia\\
 Department of Statistics, University of British Columbia\\
  BC Cancer Research, Part of Provincial Health Care Authority\\
  Vancouver, British Columbia, Canada \\
  \texttt{ypp@stat.ubc.ca} 
   \And
 Khanh Dao Duc\thanks{Corresponding author.} \\
 Department of Mathematics, University of British Columbia\\
 Department of Computer Science, University of British Columbia\\
  Vancouver, British Columbia, Canada \\
  \texttt{kdd@math.ubc.ca} \\
}

\begin{document}
\maketitle
\begin{abstract}
Fairness testing evaluates whether a model satisfies a specified fairness criterion across different groups, yet most research has focused on classification models, leaving regression models underexplored. This paper introduces a framework for fairness testing in regression models, leveraging Wasserstein distance to project data distribution and focusing on expectation-based criteria. Upon categorizing fairness criteria for regression, we derive a  Wasserstein projection test statistic from dual reformulation, and derive asymptotic bounds and limiting distributions, allowing us to formulate both a hypothesis-testing procedure and an optimal data perturbation method to improve fairness while balancing accuracy. Experiments on synthetic data demonstrate that the proposed hypothesis-testing approach offers higher specificity compared to permutation-based tests. To illustrate its potential applications, we apply our framework to two case studies on real data, showing (1) statistically significant gender disparities that appear on student performance data across multiple models, and (2) significant unfairness between pollution areas under multiple fairness criteria affecting housing price data, robust to different group divisions, with feature-level analysis identifying spatial and socioeconomic drivers.
\end{abstract}


\section{Introduction}
\RV{Fairness in machine learning models has become a critical concern as predictive models are increasingly deployed in high-stakes domains such as hiring~\citep{smelyakov2023analysis}, lending~\cite{durojaiye2024designing}, and healthcare~\citep{alanazi2022using}. 
Fairness testing formally evaluates whether a model satisfies a specified fairness criterion, providing guarantees about equitable performance across different groups~\cite{chen2024fairness}. One way to conduct fairness testing is statistical testing, which leverages hypothesis testing frameworks to quantify whether observed disparities in model predictions are statistically significant rather than due to random variation. While statistical fairness testing has been extensively studied in classification tasks~\citep{taskesen2021statistical,si2021testing,chen2025fairness}, it remains underexplored for regression models~\citep{chen2024fairness}. This gap is significant because regression models produce continuous outputs where unfairness presents differently than in binary predictions, requiring specialized statistical frameworks for testing.}

In this work, we categorize fairness criteria in regression into several classes, with a focus on expectation-based fairness criteria, which require that model outputs (or errors) have equal expected values across different population groups. Within this context, we build on the framework proposed in~\cite{taskesen2021statistical}, apply this distance to fairness testing in regression models, and extend the theoretical foundations to the regression setting. We explore the utility of this distance in two main applications: (1) fairness testing, where we develop a hypothesis-testing framework to test the fairness of regression models, and (2) optimal data perturbation, a natural extension of our testing framework, where we propose a procedure to adjust data toward greater fairness, with some trade-off in accuracy. We perform synthetic experiments to benchmark Wasserstein projection-based fairness testing against permutation-based fairness testing. We further conduct two case studies: (1) We apply our testing framework to assess fairness on Math and Portuguese grades with respect to gender on a dataset that describes student performance in Math and Portuguese grades with respect to gender; (2) We use our framework to test \RV{fairness} between low and high-pollution areas on a dataset that contains housing and environmental features including pollution levels, \RV{examine whether the fairness decisions are robust to different percentile-based group divisions}, followed by a feature-level analysis using the optimal data perturbation method.


\section{Background and related works}
\subsection{Fairness criteria for regression models}
\label{sec:fairness-reg}


We introduce the following notations. Let $\mathcal{R}: \mathcal{X} \rightarrow \mathcal{Y}$ be a regressor, where $\mathcal{X}$ is the feature space and $\mathcal{Y}$ is the label space. Let $A \in \mathcal{A}$ denote a binary sensitive attribute (e.g., gender), where $\mathcal{A}$ represents the sensitive attribute space. Let $X \in \mathcal{X}$ denote the features. Let $Y \in \mathcal{Y}$ denote a numerical label. Let $\mathbb{P} \in \mathcal{P}$ denote the joint true population distribution governing $(X,A,Y)$ where $\mathcal{P}$ represents the space of all distributions on \RV{$(\mathcal{X},\mathcal{A},\mathcal{Y})$}.

\begin{table*}[t]

\small 
\centering
\caption{Fairness criteria for regression models}
\small
\label{tab:fairness_criteria_regression}
\begin{tabular}{>{\raggedright\arraybackslash}p{2.5cm} >{\raggedright\arraybackslash}p{4.5cm} >{\raggedright\arraybackslash}p{3cm} >{\raggedright\arraybackslash}p{5cm}}
\toprule
\textbf{Criterion Name} & \textbf{Expression} & \textbf{Reference} & \textbf{Use Case} \\
\midrule
Statistical Parity & $\mathbb{P}(\mathcal{R}(X) \mid A{=}1) \overset{d.}{=} \mathbb{P}(\mathcal{R}(X) \mid A{=}0)$ & \cite{fitzsimons2019general,chzhen2022minimax,agarwal2019fair} & Ensures predictions are independent of sensitive attribute. \\
\addlinespace
Error Parity & $\mathbb{P}(E(Y, \hat{Y}) \mid A{=}0) \overset{d.}{=} \mathbb{P}(E(Y, \hat{Y}) \mid A{=}1)$ & \cite{gursoy2022error} & Ensures prediction errors are similarly distributed across groups. \\
\addlinespace
Equal Mean & $\mathbb{E}[\mathcal{R}(X) \mid A{=}0] = \mathbb{E}[\mathcal{R}(X) \mid A{=}1]$ & \cite{fitzsimons2019general} & Equality of average predictions for fairness. \\
\addlinespace
Accuracy Parity & $\mathbb{E}[E(Y, \mathcal{R}(X)) \mid A{=}1] = \mathbb{E}[E(Y, \mathcal{R}(X)) \mid A{=}0]$ & \cite{chi2021understanding} & Equal predictive accuracy across groups. \\
\addlinespace
Bounded Group Loss & $\forall a, \mathbb{E}[l(Y, \mathcal{R}(X)) \mid A{=}a] \leq \epsilon_a$ & \cite{agarwal2019fair} & Model loss must stay under a threshold for each group. \\
\bottomrule
\label{tab:fairness-criterion}
\end{tabular}
\end{table*}

\Cref{tab:fairness_criteria_regression} summarizes commonly adopted fairness criteria for regression. These criteria vary in terms of their mathematical formulation and the fairness objectives they aim to achieve. Some criteria focus on aligning model outputs across groups (e.g., Statistical Parity, Equal Mean), while others emphasize equality in prediction quality or error (e.g., Error Parity, Accuracy Parity, Bounded Group Loss). The appropriate choice of fairness criterion depends on the application context and legal or ethical considerations.

\subsection{\RV{Fairness measurements, fairness-aware training algorithms and statistical fairness testing for regression models}}
\RV{Research on fairness in regression has primarily focused on defining appropriate fairness measurements and developing fairness-aware model training algorithms. Various measurements have been proposed to quantify fairness violations in regression models. The most common approach uses difference-based metrics (e.g., mean difference), which quantify the absolute difference between groups for a given fairness criterion~\citep{calders2013controlling,germino2025intersectional}. Beyond this, Steinberg et al. introduced probabilistic classification-based measures that capture distributional differences in predictions across groups~\cite{steinberg2020fairness}, and Suárez Ferreira et al. developed a systematic framework that helps practitioners measure and compare fairness across different approaches to regression models~\cite{ferreira2025general}. Alternatively, multiple approaches have been developed to enable fairness-aware training in regression. In-processing techniques incorporate fairness constraints during model training~\cite{chai2022fairness,chzhen2020fair,shah2022selective,berk2017convex,komiyama2018nonconvex}. Post-processing methods modify predictions to satisfy fairness criteria~\cite{xian2024differentially,chzhen2020fairwasserstein}. In addition, Agarwal et al.~\cite{agarwal2019fair} proposed a reduction-based framework that converts fair regression into a sequence of cost-sensitive classification problems, enabling the use of existing fairness-aware classification algorithms. 

The above fairness measurements and fairness-aware training algorithms are designed for model evaluation and improvement rather than formal fairness testing, and thus cannot provide decisions on whether observed fairness violations are statistically significant or whether models require unfairness mitigation. To our knowledge, methods for statistically testing regression fairness remain limited. Non-parametric permutation tests~\citep{diciccio2020evaluating}, and distributional tests such as the k-sample Anderson-Darling test~\citep{scholz1987k, gursoy2022error}, have been explored, but these methods either suffer from scalability issues or are difficult to generalize across fairness criteria in regression contexts.}

\subsection{\RV{Statistical fairness testing for classification models}}
\RV{Statistical fairness testing for classification models employs either permutation-based tests or the Wasserstein projection-based tests. Permutation-based tests, similar to those used in regression settings, suffer from scalability issues that limit their practical applicability~\cite{diciccio2020evaluating}.} 

Wasserstein projection-based tests leverage the geometry of the feature space to compare empirical distributions across demographic groups with a reference distribution that represents a fair model~\citep{taskesen2021statistical,si2021testing,chen2025fairness}. Embedded within a statistical hypothesis testing framework, one can evaluate whether a classifier's predictions differ significantly across groups, thereby testing the significance with respect to the given fairness criterion. This type of method is particularly useful for formalizing fairness as a testable hypothesis, where the null hypothesis typically asserts that a classifier treats groups (e.g., male and female) similarly under a chosen fairness criterion. Unlike Boolean fairness checks, this statistical approach provides a quantifiable measure of deviation from fairness. However, existing implementations of Wasserstein-based fairness tests are often computationally demanding and limited to binary classification settings. 
In this paper, we extend the framework proposed in~\cite{taskesen2021statistical} to regression models and identify special cases in regression problems that do not encounter computational issues.

\subsection{Fairness testing paradigms other than statistical testing}
\RV{Other fairness testing research focuses on generating discriminatory instances to reveal fairness violations; these approaches employ various strategies, including random generation~\cite{fan2022explanation}, search-based techniques~\cite{harman2015achievements,harman2001search,lakhotia2007multi}, and verification-based methods~\cite{albarghouthi2017fairsquare,bastani2019probabilistic,ghosh2022algorithmic}. Another prevalent approach uses metamorphic relations as test oracles~\cite{pu2022fairness,sathiesh2021aequevox}, expecting that perturbing only sensitive attributes should not change predictions, with applications to natural language processing~\cite{dhamala2021bold,huang2020reducing,liu2020does,sheng2019woman}, computer vision~\cite{pu2022fairness}, and speech recognition~\citep{rajan2022aequevox}. While effective at finding individual violations, these methods primarily address individual fairness (whether similar individuals receive similar outcomes) and typically report the existence of discriminatory instances rather than quantifying the statistical significance of group-level disparities. In contrast, statistical fairness testing provides quantifiable statistical guarantees about group-level fairness through hypothesis testing.}

\section{Categorization of fairness criteria for regression problems}
\label{sec:catogorization}
The common fairness criteria for regression problems can be broadly grouped into two categories: expectation-based fairness and distribution-based fairness.

Let $d(Y, \mathcal{R}(X))$ denote a discrepancy function between the true label $Y$ and the model prediction $\mathcal{R}(X)$. Let $\hat{p}_a^N$ denote the empirical marginal probability when the sensitive attribute $A = a$.

Expectation-based criteria assess whether model predictions or errors have similar expected values (means) across groups.  It includes two subtypes: exact expectation equivalence and expectation equivalence within a tolerance, defined as follows.

\begin{definition}[Exact Expectation Equivalence]
\label{def:exact-expectation-equivalence}
We say a model satisfies \emph{exact expectation equivalence} if:
\[
\mathbb{E}_{\mathbb{P}}[\phi(X, A, Y)] = 0,
\]
where
\[
\phi(X, A, Y) = \frac{d(Y, \mathcal{R}(X)) \cdot \mathbbm{1}_{A = 1}}{\hat{p}_1^N} - \frac{d(Y, \mathcal{R}(X)) \cdot \mathbbm{1}_{A = 0}}{\hat{p}_0^N}.
\]
\end{definition}

\begin{definition}[Expectation Equivalence Within a Tolerance]
\label{def:approx-expectation-equivalence}
We say a model satisfies \emph{expectation equivalence within a tolerance} if:
\[
\forall a \in \mathcal{A}, \quad \mathbb{E}_{\mathbb{P}}[\phi(X, A, Y)] \leq 0,
\]
where
\[
\phi(X, A, Y) = \frac{d(Y, \mathcal{R}(X)) \cdot \mathbbm{1}_{A = a}}{\hat{p}_a^N} - \varepsilon_a,
\]
and $\varepsilon_a \geq 0$ is a user-specified tolerance for group $a$.
\end{definition}

\paragraph{Examples:}
\begin{itemize}
    \item For \emph{exact expectation equivalence}, common choices for the discrepancy function $d$ include:
    \begin{itemize}
        \item Equal mean: $d(Y, \mathcal{R}(X)) = \mathcal{R}(X)$ 
        \item Accuracy parity: $d(Y, \mathcal{R}(X)) = E(Y, \mathcal{R}(X))$ where $E$ denotes an error function (e.g., absolute error $|Y - \mathcal{R}(X)|$)
    \end{itemize}
    \item For \emph{expectation equivalence within a tolerance}, examples include:
    \begin{itemize}
        \item Bounded group loss: $d(Y, \mathcal{R}(X)) = l(Y, \mathcal{R}(X))$ 
        \item Generalized $\epsilon$-fairness~\citep{taskesen2021statistical,si2021testing}
    \end{itemize}
\end{itemize}

Distribution-based criteria evaluate fairness by comparing the full distributions of predictions across groups, capturing differences in both mean and distribution shape. It includes exact equivalence and equivalence within a tolerance, defined as follows.

\begin{definition}[Exact Distributional Equivalence]
\label{def:exact-distributional-equivalence}
A model satisfies \emph{exact distributional equivalence} if:
\[
\mathbb{P}(d(Y, \mathcal{R}(X)) \mid A = 0) \overset{d.}{=} \mathbb{P}(d(Y, \mathcal{R}(X)) \mid A = 1),
\]
where $\overset{d.}{=}$ denotes equality in distribution. This criterion requires that the conditional distributions of the discrepancy function be identical across groups. 
\end{definition}

\begin{definition}[Distributional Equivalence Within a Tolerance]
\label{def:approx-distributional-equivalence}
A model satisfies \emph{distributional equivalence within a tolerance} if:
\[
h\left(\mathbb{P}(d(Y, \mathcal{R}(X)) \mid A = 0), \mathbb{P}(d(Y, \mathcal{R}(X)) \mid A = 1)\right) \leq \varepsilon,
\]
where $h$ is a divergence or distance function that quantifies the difference between the conditional distributions, and $\varepsilon \geq 0$ is a user-defined tolerance level.
\end{definition}

\paragraph{Examples:}
\begin{itemize}
    \item For \emph{exact distributional equivalence}, examples include statistical parity and error parity. 
    \item For \emph{distributional equivalence within a tolerance}, examples include  total variation and Kolmogorov–Smirnov  fairness~\citep{chzhen2022minimax}.
\end{itemize}

\section{Wasserstein projection-based fairness testing for regression models} 
\label{sec:thm-fairness-testing}
In this section, we present our Wasserstein projection-based fairness testing framework for regression models focusing on expectation-based criteria. Although expectation-based criteria may not capture all distributional disparities, they are generally more interpretable and widely used compared to distribution-based criteria, which involve more complex comparisons of full distributions~\citep{dixon2018primer,meyners2012equivalence}. We present the construction of a test statistic (\Cref{sec:construction}), its computation (\Cref{sec:computation}), and the derivation of an asymptotic upper bound (\Cref{sec:limiting-distribution}). Proofs can be found in~\Cref{sec:proofs}.


Within expectation-based criteria, we note that exact expectation equivalence can be considered a special case of expectation equivalence within a tolerance; it can be recovered from expectation equivalence within a tolerance by using the $\phi$ function from the exact expectation equivalence definition and setting $\varepsilon_a = 0$. Therefore, we focus on the more general case of expectation equivalence within a tolerance. We discuss the exact expectation equivalence case only when it leads to different results.

\subsection{Construction of the test statistic}
\label{sec:construction}
We construct the hypothesis as:
$$\mathcal{H}_0: \textit{ the regressor } \mathcal{R} \textit{ is fair},$$ against the alternative hypothesis:
$$\mathcal{H}_1: \textit{ the regressor } \mathcal{R} \textit{ is not fair}.$$
Let $\mathbb{Q}$ be any joint distribution in the space of $\mathcal{P}$. We define the set of fair distributions with respect to $\mathcal{R}$ as
\begin{align*}
    \mathcal{F}_{\mathcal{R}} &= \{\mathbb{Q} \in \mathcal{P}: \mathcal{R} \textit{ is fair relative to } \mathbb{Q}\}.
\end{align*}
We can reinterpret the hypothesis test as:
$$\mathcal{H}_0: \mathbb{P} \in \mathcal{F}_\mathcal{R}, \textit{ } \mathcal{H}_1: \mathbb{P} \notin \mathcal{F}_\mathcal{R}.$$
To measure how far the observed data distribution \(P\) deviates from fairness, we use the Wasserstein distance under a cost function \(c\) defined on tuples \((x, a, y), (x', a', y') \in \mathcal{X} \times \mathcal{A} \times \mathcal{Y}\):
\begin{equation}
\begin{aligned}
c((x, a, y), (x', a', y')) = \alpha \|x - x'\| 
   + \infty \cdot |a - a'| 
   + \beta |y - y'|,
\end{aligned}
\label{eq:cost-function}
\end{equation}
where \(\alpha, \beta \geq 0\), \(\|\cdot\|\) denotes the Euclidean norm, and the infinite cost between samples with different sensitive attribute values enforces that no mass is transported across groups. This reflects an assumption of absolute trust in the integrity of the sensitive attribute, as discussed in~\citep{taskesen2020distributionally,xue2020auditing}. 
Using $c$, the Wasserstein distance between distributions \(\mathbb{P}\) and \(\mathbb{Q}\) with respect to it is:

\[
W_c(\mathbb{P}, \mathbb{Q}) = \inf_{\pi \in \Pi(\mathbb{P}, \mathbb{Q})} \int_{\mathcal{X} \times \mathcal{X}} c(x, y) \, d\pi(x, y),
\]
where \(\Pi(\mathbb{P}, \mathbb{Q})\) is the set of all joint couplings with marginals \(\mathbb{P}\) and \(\mathbb{Q}\).
Then, we construct the test statistic \(\mathcal{T}\) as:
\begin{equation}
    \mathcal{T} := \inf_{\mathbb{Q} \in \mathcal{F}_\mathcal{R}} W_c^2(\mathbb{P}, \mathbb{Q}).
\end{equation}
The statistic \(\mathcal{T}\) captures the minimum Wasserstein distance between the observed distribution \(\mathbb{P}\) and a constrained set of distributions \(\mathcal{F}_\mathcal{R}\), providing a way to evaluate the fairness criteria encoded in the choice of cost function $c$ and constraint set $\mathcal{F}_\mathcal{R}$.
As $\mathbb{P} \in \mathcal{F}_\mathcal{R}$ if and only if $\mathcal{T} = 0$, we can reinterpret the hypothesis as:
$$\mathcal{H}_0: \mathcal{T} = 0, \textit{ } \mathcal{H}_1: \mathcal{T} > 0.$$

\subsection{Computation of the test statistic}
\label{sec:computation}

Computing $\mathcal{T}$ involves solving an optimization over an infinite-dimensional space. To transform the problem into an optimization problem over a finite space, we reformulate the problem using duality theory.

We define the following notations. Let $(\hat{X}, \hat{A}, \hat{Y})$ denote the empirical distribution of $(X,A,Y)$. Let $\hat{\mathbb{P}}^N$ denote joint distribution of $(\hat{X}, \hat{A}, \hat{Y})$. Let $\hat{p}^N \in \mathbb{R}_{++}^{|\mathcal{A}|}$ be the vector of empirical marginals of $A$, \RV{and $\hat{p}^N_a$ denote denote the empirical marginal when $A=a$.} \RV{We define $\mathcal{F}_\mathcal{R}(\hat{p}^N)$ as the set of marginally constrained fair distributions, $$\mathcal{F}_\mathcal{R}(\hat{p}^N) = \{\mathbb{Q} \in \mathbb{P}: \mathcal{R} \text{ is fair to }\mathbb{Q} \text{ and } \mathbb{Q}(A=a) = \hat{p}^N_a, \forall a \in \mathcal{A}\}.$$}
We transform the optimization space into a more-constrained probability space using the following theorem.
\begin{theorem}
\label{thm:marginal}
Suppose $\mathbb{Q} \in \mathcal{F}_\mathcal{R}$ satisfies $W_c^2(\hat{\mathbb{P}}^N, \mathbb{Q}) < \infty$, then $$\mathop{inf}\limits_{\mathbb{Q} \in \mathcal{F}_\mathcal{R}} W_c^2(\hat{\mathbb{P}}^N, \mathbb{Q}) = \mathop{inf}\limits_{\mathbb{Q} \in \mathcal{F}_\mathcal{R}(\hat{p}^N)} W_c^2(\hat{\mathbb{P}}^N, \mathbb{Q}).$$
\end{theorem}
\Cref{thm:marginal} suggests that it is sufficient to consider the Wasserstein projection onto the marginally-constrained set of fair distributions $\mathcal{F}_\mathcal{R}(\hat{p}^N).$ Hence, we can rewrite the test statistic $\mathcal{T}$ as:
\begin{align}
\mathcal{T} = 
\left\{
    \begin{aligned}
\mathop{inf}\limits_{\mathbb{Q}} & \textit{ }W_c^2(\hat{\mathbb{P}}^N, \mathbb{Q}) \\
\textit{s.t. }& \mathbb{E}_\mathbb{Q}[\phi(X, A, Y)] \leq 0 \\
& \forall a \in \mathcal{A}, \mathbb{E}_\mathbb{Q}[\mathbbm{1}_a(A)] = \hat{p}_a^N
\end{aligned}
    \right..
    \label{eq:test_stat}
\end{align}
Though constrained, the optimization problem in~\Cref{eq:test_stat} is still over an infinite-dimensional probability space. We further reformulate~\Cref{eq:test_stat} as a finite-dimensional optimization using the following theorem.
\begin{theorem}[Dual reformulation]
Let $\mathcal{X}$ and $\mathcal{Y}$ denote the Euclidean space for $X$ and $Y$. Let $(\hat{x}_i, \hat{a}_i, \hat{y}_i)$ denote i.i.d. samples from $\mathbb{\hat{P}}^N$. Then, 
\begin{widetext}
\begin{equation}
    \mathcal{T} = 
    \frac{1}{N} \underbrace{\mathop{sup}\limits_{\gamma \in \mathbb{R}} \sum_{i = 1}^N \underbrace{\mathop{inf}\limits_{\substack{x_i \in \mathcal{X} \\ y_i \in \mathcal{Y}}} \gamma \phi(x_i, \hat{a}_i, y_i) + (\alpha \| x_i - \hat{x}_i\| + \beta |y_i - \hat{y}_i|)^2}_{\text{inner minimization}}}_{\text{outer maximization}}.
\end{equation}
\end{widetext}
\label{thm:dual}
\end{theorem}
\Cref{thm:dual} asserts that computing the squared projection $\mathcal{T}$ is equivalent to solving a finite dimensional problem. The difficulty of solving~\Cref{thm:dual} depends on the structure of $\phi$. 
Furthermore, the computation can become significantly more tractable when analytical forms of $\mathcal{R}$ and $d$ are available. For example, in the special case of equal mean fairness criteria (see~\Cref{sec:catogorization}) and of linear regression, the dual objective have a closed-form solution.

\begin{corollary}[Special case of~\Cref{thm:dual}]
\label{cor:dual}
    Suppose $\mathcal{R}(x) = \rho^T x+\sigma$ (i.e., linear regression) and $d(y, \hat{y}) = \hat{y}$ (i.e., equal mean), and $\alpha=1$, Then,    
    $$\mathcal{T} = \frac{(\sum_{i=1}^N \lambda(\hat{a}_i)(\rho^T \hat{x}_i + \sigma))^2}{N \|\rho\|^2(\sum_{i=1}^N \lambda(\hat{a}_i)^2)},$$
    where $\lambda(a) = (\hat{p}_1^N)^{-1}\mathbbm{1}_1(a) - (\hat{p}_0^N)^{-1}\mathbbm{1}_0(a).$
\end{corollary}

\subsection{Derivation of an asymptotic upper bound}
\label{sec:limiting-distribution}
We describe the asymptotic behavior of $\mathcal{T}$ with the following theorem.
\begin{theorem}[Asymptotic upper bound]
\label{thm:asymptotic_upper_bound}
Assume the gradient of the regressor, $\nabla_X \mathcal{R}(x)$, is locally Lipschitz continuous. Under the null hypothesis $\mathcal{H}_0$ with the fairness criterion belonging to expectation equivalence within a tolerance, and $\alpha=1$, we have
$$N \times \mathcal{T} \lesssim_D \theta \chi_1^2,$$
where $\lesssim_D$ denotes a distributional upper bound~\citep{shorack2009empirical}, and $\chi_1^2$ is the chi-square distribution with 1 degree of freedom, and 

\begin{equation*}
    \begin{aligned}
    &\theta = \frac{\mathop{Cov}(Z')}{\mathbb{E}_{\mathbb{P}} \left\|\triangledown_X d(\mathcal{R}(X), Y) \left(\frac{\mathbbm{1}_1(A)}{p_1} - \frac{\mathbbm{1}_0(A)}{p_0}\right) \right\|^2} ,\\
    &\text{where $\triangledown_X$ denotes the derivative with respect to $X$},\\[10pt]
    &
    \begin{aligned}
    Z' = \frac{1}{p_0 p_1} \Bigg\{ d(Y, \mathcal{R}(X))(p_0\mathbbm{1}_1(A) - p_1\mathbbm{1}_0(A)) + \mathbbm{1}_0(A) \mathbb{E}_\mathbb{P}[d(Y, \mathcal{R}(X)) \mathbbm{1}_1(A)] - \mathbbm{1}_1(A) \mathbb{E}_\mathbb{P}[d(Y, \mathcal{R}(X)) \mathbbm{1}_0(A)] \Bigg\}.
    \end{aligned}
    \\[5pt]
\end{aligned}
\end{equation*}
\label{thm:limiting_extension}
\end{theorem}
By~\Cref{thm:limiting_extension}, we know the asymptotic upper bound of $\mathcal{N} \times \mathcal{T}$ follows a chi-square distribution. Since $\theta$ in the limiting distribution depends on $\mathbb{P}$, we need to derive a consistent estimator for $\theta$ so that we can use the limiting distribution for hypothesis testing. By the law of large numbers, the denominator in $\theta$ can be estimated by the sample average, that is,
\begin{equation*}
\begin{aligned}
    &\reallywidehat{\mathbb{E}_{\mathbb{P}} \|\triangledown_X d(\mathcal{R}(X), Y)(\frac{\mathbbm{1}_1(A)}{p_1} - \frac{\mathbbm{1}_0(A)}{p_0})\|^2} = \\
    &\frac{1}{N}\sum_{i=1}^N \|\triangledown_X d(\mathcal{R}(\hat{x}_i), \hat{y}_i)\left(\frac{\mathbbm{1}_1(\hat{a}_i)}{\hat{p}^N_1} - \frac{\mathbbm{1}_0(\hat{a}_i)}{\hat{p}^N_0}\right)\|^2, \text{ and}
\end{aligned}
\end{equation*}

\begin{equation*}
\begin{aligned}
    (\reallywidehat{\mathop{Cov}(Z')}) &=  \big\{\frac{1}{N\hat{p}_0 \hat{p}_1} \sum_{i=1}^N[\{d(\hat{y}_i, \mathcal{R}(\hat{x}_i))(p_0\mathbbm{1}_1(\hat{a}_i) - p_1\mathbbm{1}_0(\hat{a}_i))\}]   + \frac{1}{N} \mathbbm{1}_0(\hat{a}_i)  \sum_{j=1}^N [d(\hat{y}_i, \mathcal{R}(\hat{x}_i)) \mathbbm{1}_1(\hat{a}_i)]\\
    & - \frac{1}{N} \mathbbm{1}_1(\hat{a}_i) \sum_{j=1}^N[d(Y, \mathcal{R}(\hat{x}_i)) \mathbbm{1}_0(\hat{a}_i)]\big\}^2.
\end{aligned}
\end{equation*}

\begin{figure}[t]
    \centering
    \includegraphics[width=0.25\linewidth]{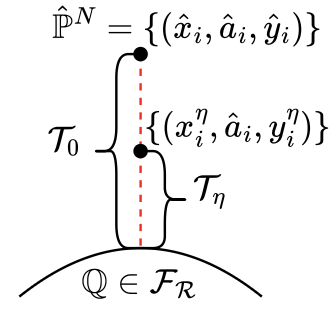}
    \caption{Optimal data perturbation. Starting from $\hat{\mathbb{P}}^N$, perturbed distributions $\{(x_i^\eta, \hat{a}_i, y_i^\eta)\}$ move along the projection path toward the fairness set $\mathcal{F}_R$, with $\eta=1$ achieving full fairness.}
    \label{fig:perturbation_illustration}
\end{figure}

When the fairness criterion belongs to exact expectation equivalence (see~\Cref{sec:catogorization}), we can tighten the asymptotic upper bound of $\mathcal{T}$ to a limiting distribution using the following theorem. 
\begin{theorem}[Limiting distribution]
Assume the gradient of the regressor, $\nabla_X \mathcal{R}(x, y)$, is locally Lipschitz continuous. Under the null hypothesis $\mathcal{H}_0$, with the fairness criterion belonging to exact expectation equivalence, and $\alpha=1$, we have $\theta$ and $Z'$ are defined the same as in~\Cref{thm:asymptotic_upper_bound}.
\label{thm:limiting}
\end{theorem}


In the special case of linear regression and equal mean fairness criterion, we can have a simplified solution for computing the coefficient $\theta$.
\begin{corollary}  
[Special case of~\Cref{thm:limiting}]
    Suppose $\mathcal{R}(x) = \rho^T x+\sigma$ (i.e., linear regression) and $d(y, \hat{y}) = \hat{y}$ (i.e., equal mean criterion), the estimate of $\theta$ (i.e. $\hat{\theta}$ in~\Cref{thm:limiting}) can be simplified as,
    $$\hat{\theta} = \frac{\sum_{i=1}^N (\rho \hat{x}_i+\sigma)^2 \{\frac{\mathbbm{1}_1(\hat{a}_i)}{(\hat{p}_1^N)^2} + \frac{\mathbbm{1}_0(\hat{a}_i)}{(\hat{p}_0^N)^2}\}} {\|\rho\|^2\sum_{i=1}^N \{\frac{\mathbbm{1}_1(\hat{a}_i)}{(\hat{p}_1^N)^2} + \frac{\mathbbm{1}_0(\hat{a}_i)}{(\hat{p}_0^N)^2}\}}.$$
    \label{cor:special-case-limiting-distribution}
\end{corollary}
Numerical simulations illustrating~\Cref{cor:special-case-limiting-distribution} can be found in~\Cref{sec:validation-limiting-distribution}.

\section{Wasserstein projection-based optimal data perturbation for regression models}
In this section, we present an optimal data perturbation framework, as an extension coming from the fairness testing theorems, to improve fairness with respect to predefined fairness criteria by a degree, as formalized in the following theorem. 


\begin{theorem}[Optimal data perturbation]
    Let  $\delta_i = \phi(\hat{x}_i, \hat{a}_i, \hat{y}_i$),  which is the value of the standard fairness metric for each data point $i$. 
    Let 
    \begin{widetext}
    \begin{equation}
    \mathcal{T}_\eta = 
    \frac{1}{N} \mathop{sup}\limits_{\gamma \in \mathbb{R}} \sum_{i = 1}^N \underbrace{\mathop{inf}\limits_{\substack{x_i \in \mathcal{X} \\ y_i \in \mathcal{Y}}} \gamma (\phi(x_i, \hat{a}_i, y_i)+\eta\delta_i) 
    + (\alpha \| x_i - \hat{x}_i\| + \beta |y_i - \hat{y}_i|)^2}_{\text{inner minimization}},
    \label{eq:optimal-perturbation}
    \end{equation}
    \end{widetext}
where $\eta \in [0,1]$. Let $x_i^\eta \in \mathcal{X}$ and $y_i^\eta \in \mathcal{Y}$ be the optimal solutions from the inner minimization in~\Cref{eq:optimal-perturbation}. Then, the empirical distribution generated by $\mathcal{R}$ and $\{x_i^\eta, \hat{a}_i, y_i^\eta\}_{i=1}^N$ perturbs the model predictions such that the value of the standard fairness metric is reduced by $\eta$.
    \label{thm:optimal-perturbation}
\end{theorem}


\RV{\Cref{fig:perturbation_illustration} provides an illustration of \Cref{thm:optimal-perturbation}}, with proof in~\Cref{sec:proofs}. Using the theorem, we use $\mathcal{R}(x_i^\eta)$ to get the perturbed model predictions. In a special case of linear regression and equal mean fairness criterion, we have closed-form solutions for $x_i^\eta$ using the following corollary, with proof in~\Cref{sec:proofs}) and validation example  in~\Cref{sec:optimal-perturbation-special-case}..
\begin{corollary}[Optimal data perturbation for the special case]
     Suppose $\mathcal{R}(x) = \rho^T x+\sigma$ (i.e., linear regression) and $d(y, \hat{y}) = \hat{y}$ (i.e., equal mean criterion), and $\alpha = 1$. Then $x_i^\eta$ and $y_i^\eta$ in~\Cref{thm:optimal-perturbation} is given by $x_i^\eta = \hat{x}_i - \frac{1}{2}\eta\rho^T\gamma^* \lambda(\hat{a}_i)$ and $y_i^\eta = \hat{y}_i$, where $$\gamma^* = \frac{2\rho^T \sum_{i=1}^N \lambda(\hat{a}_i) \hat{x}_i + 2\sigma \sum_{i=1}^N \lambda(\hat{a}_i)}{\|\rho\|^2 \sum_{i=1}^N \lambda(\hat{a}_i)^2}.$$ 
     \label{cor:optimal-perturbation-special}
\end{corollary}

\section{Implementation}
We implement the framework in Python using NumPy, SciPy, and scikit-learn. The nested optimization in Theorem~\ref{thm:dual} is solved using BFGS\citep{dai2002convergence} for the inner minimization and bounded scalar optimization~\citep{nocedal1999numerical} for the outer maximization over $\gamma$ within a bounded range determined empirically and a tolerance of $10^{-3}$. We set cost function weights $\alpha = 1$ and $\beta = 0$ for equal mean criterion or $\beta = 1$ for squared error criterion. The gradients in the p-value computation in~\Cref{thm:limiting} are approximated via SHAP values~\citep{lundberg2017unified} for tree-based models and finite differences for SVR~\citep{nocedal1999numerical}. Data and code are available at: \url{https://anonymous.4open.science/r/Wasserstein-projection-0FD7/}.

\section{Synthetic experiments to benchmark Wasserstein projection-based fairness testing}
\begin{figure*}[t]
    \centering
    \includegraphics[width=0.75\linewidth]{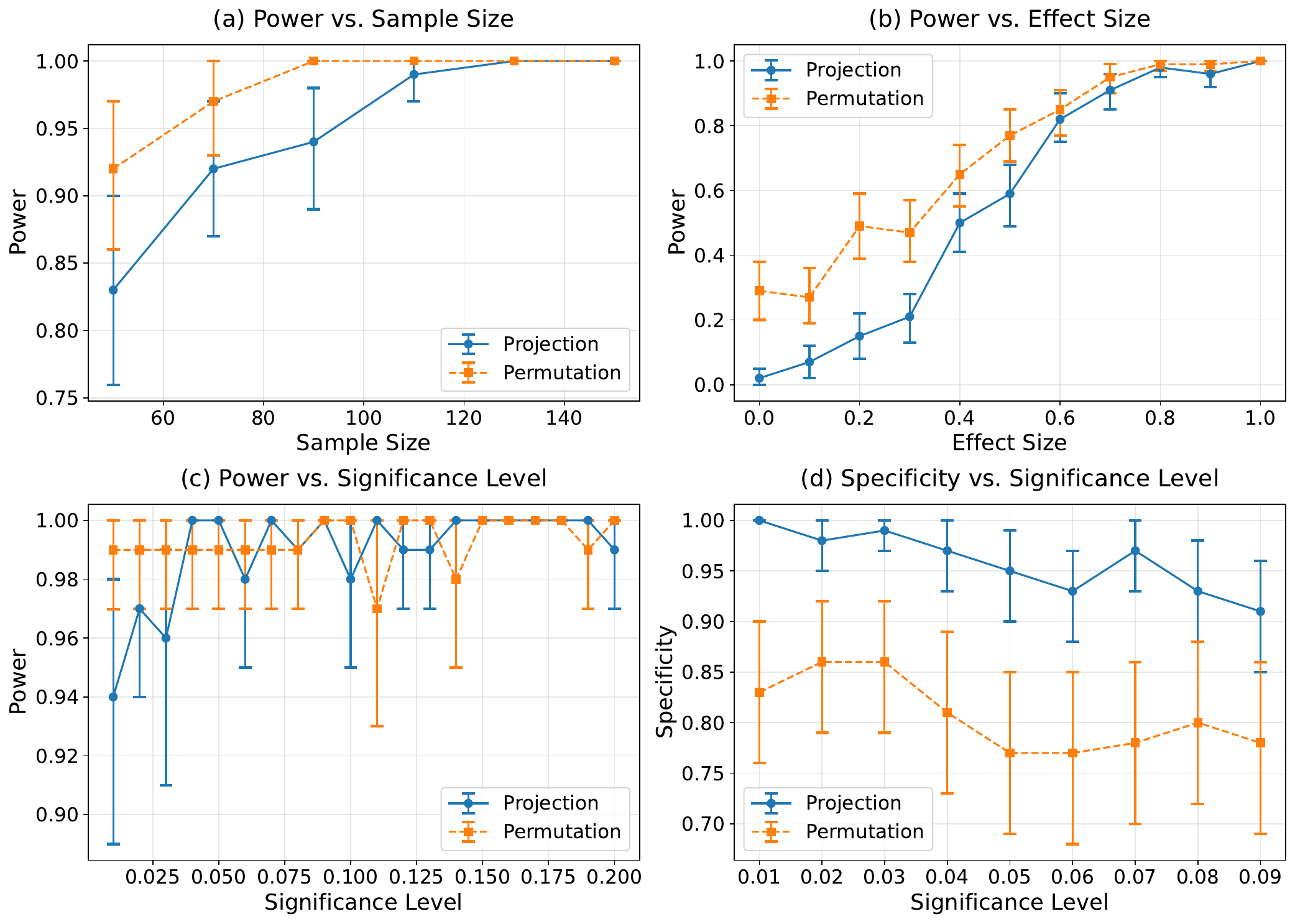}
    \caption{\RV{Visualization of simulations to examine the relationship between (a) Power vs. Sample size, (b) Power vs. Effect size, (c) Power vs. Significance level and (d) Specificity vs. Significance level. The bars denote 95\% CI.}}
    \label{fig:power-specificity}
\end{figure*}

We benchmarked Wasserstein projection-based fairness testing against permutation-based testing under linear regression with the equal mean criterion.
We evaluated power, specificity \RV{and their 95\% confidence intervals (CI) using non-parametric bootstrapping} across varying sample sizes, effect sizes, and significance levels.
Power refers to the test's ability to correctly reject a false null hypothesis, and is influenced by factors such as effect size (e.g.,  the magnitudes of differences in model outcomes across demographic groups in fairness testing), sample size, and the chosen significance level~\citep{lakens2013calculating,cohen2013statistical,ellis2010essential}. In contrast, specificity concerns the probability of correctly failing to reject a true null hypothesis,  and is primarily governed by the significance level~\citep{lehmann2005testing}.

\begin{table}[h]
\centering
\caption{Comparison of average computation time per experiment in seconds between Wasserstein projection-based tests and permutation tests in various simulation setups.}
\label{tab:timing_comparison}
\begin{tabular}{lcc}
\toprule
Setup & Wasserstein  (s) & Permutation (s)  \\
\midrule
Power vs Effect Size & 2.66 & 38.02  \\
Power vs Sample Size & 2.65 & 34.11  \\
Power vs Significance Level & 2.69 & 38.92  \\
Specificity vs Significance Level & 2.64 & 32.62 \\
\midrule
\textbf{Overall Average} & \textbf{2.66} & \textbf{35.92}  \\
\bottomrule
\end{tabular}
\end{table}

\begin{table*}[!htbp]
\small
\centering
\begin{tabular}{|l|ccc|ccc|}
\hline
\textbf{Model} & \multicolumn{3}{c|}{\textbf{Math}} & \multicolumn{3}{c|}{\textbf{Portuguese}} \\
\hline
 & relative MAE & mean difference & p-value &   relative MAE & mean difference  & p-value \\
\hline
Linear & 0.29 & -0.09 & 0.71   & 0.16 & 0.38 & \HL{0.02}   \\ \hline
Lasso & 0.32 & -0.02 & 0.71   & 0.20 & -0.07  & \HL{0.01}  \\ \hline
Ridge & 0.29 & -0.09 & 0.72   & 0.16 & 0.38 & \HL{0.02}   \\ \hline
SVR Linear  & 0.28 & -0.10 & 0.65  & 0.15 & 0.33 & \HL{0.03}   \\ \hline
SVR RBF & 0.31  & 0.04 & 0.65  & 0.17 & 0.25 & \HL{0.00}   \\ \hline
SVR Poly & 0.30 & 0.17 & 0.08    & 0.16 & 0.32 & \HL{0.00}  \\ 
\hline
\end{tabular}
\caption{Relative MAEs, mean differences and p-values between predicted means between females and males from different models on predicting Math and Portuguese final grades. P-values that are below the significance level (0.05) are highlighted in red.}
\label{tab:math_portuguese}
\end{table*}

\RV{In our simulation setups (see~\Cref{sec:simulation-details} for details), our method \RV{converges to the full power with} permutation testing in power once sample size exceeds 130 (\Cref{fig:power-specificity}(a)) and when effect size exceeds 0.8 (\Cref{fig:power-specificity}(b)). For \RV{smaller sample sizes and effect sizes}, permutation testing performs slightly better. With increasing significance levels, both methods achieve similar power (\Cref{fig:power-specificity}(c)), but our method consistently yields higher specificity across all levels (\Cref{fig:power-specificity}(d)). Across all experiments, the CIs for our approach are on par with permutation-based testing. Overall, our approach achieves competitive power and superior specificity compared to permutation-based testing. Moreover, in~\Cref{tab:timing_comparison}, our method demonstrates substantially better computational efficiency, with an average computation time of 2.66 seconds per experiment compared to 35.92 seconds for permutation testing across all our setups, representing approximately a 13.5-fold speedup.}

\section{A case study on student performance dataset for fairness testing}
The Student Performance dataset~\cite{cortez2008using}, collected from Portuguese secondary schools, is a common benchmark in educational data mining. It includes 649 observations with 33 variables covering demographics (e.g., age, gender), social factors (e.g., family size, parental education), and academic indicators (e.g., study time, prior grades). The response variables are final Math and Portuguese grades, scored from 0 to 20. Previous work has used both interpretable models such as Decision Trees and Linear Regression~\cite{cortez2008using}, and more complex models like Random Forests, Support Vector Machines, and k-Nearest Neighbors to capture nonlinear patterns~\cite{el2019multiple}. Motivated by evidence that performance in Math and Portuguese may reflect gender-related influences such as social expectations and teacher bias~\citep{stoet2013sex, else2010cross}, we evaluate the accuracy of several regression models and assess their fairness with respect to gender using the equal mean criterion.

To set up the experiments, we partitioned the data into two groups by the students' genders. For model training, we removed variables $G_1$ (first period grade) and $G_2$ (second period grade) since we wanted to focus on $G_3$ (final grade) prediction, and gender as it is the sensitive attribute. We trained various regression models including  Linear Regression, Lasso Regression, Ridge Regression, and Support Vector Regression (SVR) with linear, Radial Basis Function (RBF), and polynomial kernels to predict the final from the remaining explanatory variables. 

\RV{We began with a descriptive comparison of predicted outcomes across genders. As summarized in Table~\ref{tab:math_portuguese}, for Math final grades, the predicted mean differences between female and male students are small across all models, ranging from $-0.10$ to $0.17$, with no consistent directional pattern. In contrast, for Portuguese final grades, the predicted means for female students are generally higher than those for male students across most models, with mean differences ranging from $0.25$ to $0.38$. The only exception is Lasso Regression, where the mean difference is slightly negative ($-0.07$). This observation suggests that, compared to Math, predictions for Portuguese final grades tend to be more favorable to female students across a wide range of regression models.}
\RV{To assess whether these observed differences should be considered statistically unfair, we applied our Wasserstein projection-based hypothesis test under the equal mean criterion. \Cref{tab:math_portuguese} shows the corresponding $p$-values. For Math final grades, all models yield $p$-values above 0.05, indicating no statistically significant violations of the equal mean criterion. For Portuguese final grades, however, all models produce $p$-values below 0.05, indicating statistically significant violations of the equal mean criterion. 
In terms of predictive accuracy, Table~\ref{tab:math_portuguese} shows that all models achieve lower relative mean absolute errors (MAEs) for Portuguese final grades (ranging from 0.16 to 0.20) than for Math final grades (ranging from 0.28 to 0.32),  indicating better predictive performance on Portuguese. Together, these results highlight a trade-off between accuracy and fairness across fairness and accuracy.}

\section{A case study on Boston Housing dataset for \RV{fairness} testing and data perturbation}

The Boston Housing Dataset~\citep{harrison1978hedonic} is a benchmark in machine learning and statistics, originally collected to study housing prices in the Boston metropolitan area. It includes 506 observations with 14 variables such as crime rate, average rooms per dwelling, property tax rate, and proportion of Black residents. The outcome is the median value of owner-occupied homes (in \$1000s). While early studies relied on linear regression for interpretability~\citep{harrison1978hedonic}, later work adopted polynomial, Ridge, and Lasso regression to capture nonlinear effects~\citep{sanyal2022boston,li2024comparative}, and more recent research has shown that tree-based and ensemble methods such as Random Forests, Gradient Boosting, and XGBoost can achieve higher accuracy~\citep{xia2024boston,li2024comparative}. Motivated by the original finding that buyers pay a premium for clean air~\citep{harrison1978hedonic}, we investigate two questions: \RV{(i) whether predicted housing prices are statistically significantly different between areas with high and low pollution under different fairness criteria and different division points of the sensitive attribute, and (ii) which features drive the fairness behavior of the model, either compensating for high pollution or being undervalued in such areas. We emphasize that ``fairness'' in this context refers to statistical fairness as defined by the adopted criteria, rather than sociological or normative notions of fairness.}
\begin{table}[t]
\small
\centering
\begin{tabular}{|c|c|c|}
\hline  
\textbf{NOX Percentile} & \textbf{Equal mean p-value} & \textbf{Squared error p-value}\\
\hline
25th & 2.76e-12 & 9.50e-85 \\ \hline 
50th & 2.91e-18 & 6.03e-104 \\ \hline
75th & 2.33e-13 & 7.35e-128\\ \hline
\end{tabular}
\caption{\RV{Fairness testing by different division points of NOX.}}
\label{tab:fairness-testing-NOX}
\end{table}

\begin{table*}[t]
\small
\centering
\begin{tabular}{|l|c|c|}
\hline
\textbf{Ranked Features} & \textbf{High NOX} & \textbf{Low NOX} \\
\hline
lstat (\% lower status of the population) & -0.156 & 0.161\\
\hline
dis (weighted distances to five Boston employment centres) & -0.016 & 0.016 \\
\hline
chas (Charles River dummy variable (1 if tract bounds river; 0 otherwise) & 0.007 & -0.007\\
\hline
rad (index of accessibility to radial highways) & 0.006 & -0.006\\
\hline
ptratio (pupil-teacher ratio by town) & -0 & 0\\
\hline
crim (per capita crime rate by town) & -0 & 0\\
\hline
rm (average number of rooms per dwelling) & 0& -0\\
\hline
indus (proportion of non-retail business acres per town) & -0& 0\\
\hline
tax (full-value property-tax rate per \$10,000) & -0& 0 \\
\hline
age (proportion of owner-occupied units built prior to 1940) & -0 & 0 \\
\hline
zn (proportion of residential land zoned for lots over 25,000 sq.ft.) & 0& -0\\
\hline
\end{tabular}
\caption{Ranked features for high NOX and low NOX areas.}
\label{tab:ranked_features}
\end{table*}

To set up the experiments, \RV{we treated nitric oxide concentration (NOX) as the sensitive attribute and partitioned the data into two groups according to a chosen division point. For a given threshold, areas with NOX levels above the division point are classified as high-NOX areas, while those below the division point are classified as low-NOX areas. We considered multiple candidate division points corresponding to different percentiles of the NOX distribution.} We adapted the linear regression model proposed in the original study associated with the dataset~\citep{harrison1978hedonic}, \RV{as its additive structure allows the feature-level effects of the optimal data perturbation to be directly interpreted for (ii)}. Specifically, for model training, we removed the variables NOX as it is the sensitive attribute, and B (i.e., $1000(Bk - 0.63)^2$ where Bk is the proportion of blacks by town) due to concerns about its data transformation~\citep{Carlisle_2020}. We retained the remaining data transformation following Equation A.1 in~\citep{harrison1978hedonic}, and re-estimated the coefficients using least squares regression to predict the median value of owner-occupied homes. We used 0.05 as our significance level throughout the analysis.

To address (i), we applied our Wasserstein projection-based hypothesis test to \RV{the fitted linear regression model using multiple NOX division points and two expectation-based fairness criteria: the equal mean criterion and the squared error criterion. Table~\ref{tab:fairness-testing-NOX} reports the resulting $p$-values for NOX division points at the 25th, 50th (median), and 75th percentiles. Across all division points, both fairness criteria yield extremely small p-values, well below the significance level of 0.05, indicating statistically significant differences between high-NOX and low-NOX areas. The results are qualitatively consistent across division points, suggesting that the detection of unfairness is robust to the choice of NOX division threshold.}

To address (ii), we applied our model perturbation procedure proposed in~\Cref{cor:optimal-perturbation-special}  with \RV{the equal mean fairness criterion and} $\eta=1$, meaning we fully enforce equal mean predictions across the high-NOX and low-NOX areas. In~\Cref{tab:ranked_features}, we ranked features by the magnitudes of the product between the absolute values of the regression coefficients and the changes (in their transformed form) resulting from data perturbation. \RV{This quantity captures the extent and the direction to which that feature must be adjusted to achieve fairness.} We interpret these ranked features from two perspectives. 
First, the features that most compensate for residing in high-NOX areas, in descending order of contribution, are: (1) a decrease in \texttt{lstat} (percentage of lower-status population), (2) a decrease in \texttt{dis} (distance to employment centers), (3) an increase in \texttt{chas} (proximity to the Charles River), and (4) an increase in \texttt{rad} (accessibility to radial highways). Second, for low-NOX areas, these same features appear to be undervalued, in the same ranked order. In addition, among the top four contributing features, three (\texttt{dis}, \texttt{chas}, and \texttt{rad}) are spatial variables, and one (\texttt{lstat}) is a socioeconomic variable. This suggests that both socioeconomic status and spatial accessibility play roles in \RV{shaping the model's behavior} between high- and low-NOX areas.


\section{Conclusion}
In this paper, we presented the use of Wasserstein projection distance for fairness testing and optimal data perturbation on regression models under expectation-based fairness criteria. 
By extending previous work on fairness in classification to the regression setting, we addressed a significant gap in fairness evaluation for continuous prediction tasks. 
\RV{Theoretically, we developed a rigorous statistical testing framework by categorizing regression fairness criteria into expectation-based and distribution-based classes, formulating the Wasserstein projection distance as a test statistic, deriving its dual formulation to enable finite-dimensional optimization and identifying special cases with analytical solutions, and establishing its asymptotic behavior through distributional upper bounds and limiting distributions that facilitate p-value computation.}
\RV{Through synthetic experiments, we demonstrated that our Wasserstein projection-based test achieves competitive power and consistently superior specificity compared to permutation-based testing. Case studies on the Student Performance dataset revealed statistically significant gender-based disparities in Portuguese grade predictions but not in Math grade predictions. On the Boston Housing dataset, we uncovered significant unfairness in predicted home prices between high- and low-pollution areas that remained robust across different percentile-based division points. Optimal data perturbation applied to the housing predictions identified spatial features, including employment center distance, river proximity, and highway access, along with socioeconomic factors such as lower-status population percentage, as the primary contributors to model unfairness.}


In the future, we plan to pursue the following directions. First, improving the computational efficiency of the fairness metric and test statistic computation remains an important challenge, especially for large-scale or non-linear models. Second, we aim to extend our framework to support a broader class of fairness criteria, including average ratio-based definitions (see~\Cref{sec:other-criteria}). Finally, we intend to apply our methods to real-world regression problems with fairness implications, such as salary prediction across genders~\citep{crothers2010gender}, to further evaluate practical impact and relevance.

\newpage
\bibliographystyle{abbrvnat}  
\bibliography{reference}

\newpage
{\Large Appendices}

\appendix
\section{Proofs}
Before proving~\Cref{thm:marginal}, we prove the following lemma. 
\label{sec:proofs}
\begin{lemma}[Projection with marginal constraints]
\label{lemma:marginal}
Suppose $\mathbb{Q} \in \mathcal{F}_\mathcal{R}$ satisfies $W_c^2(\hat{\mathbb{P}}^N, \mathbb{Q}) < \infty$, then $\mathbb{Q} \in \mathcal{F}_\mathcal{R}(\hat{p}^N).$
\end{lemma}
\begin{proof}[Proof for Lemma~\ref{lemma:marginal}] 
As the fairness constraints are the same in $\mathcal{F}_\mathcal{R}$ and $\mathcal{F}_\mathcal{R}(\hat{p}^N)$, it suffices to verify that $\mathbb{Q} \in \mathcal{F}_\mathcal{R}$ satisfies the marginal constraint $\forall a \in \mathcal{A}, \mathbb{Q}(A = a) = \hat{p}_a^N$.

By definition of $W_c^2$, there exists a coupling $\pi$ such that
$$W_c^2(\hat{\mathbb{P}}^N, \mathbb{Q}) = \mathbb{E}_\pi [(\alpha\|X' - X\| + \infty \|A' - A\| + \beta\|Y' - Y\|)^2],$$
where the marginals of $\pi$ are $\hat{\mathbb{P}}^N$ and $\mathbb{Q}$. 

We prove by contradiction. Let $\mathbb{Q}_i$ denote the conditional distribution of $(X,A,Y)$ given $(X',A',Y') = (\hat{x}_i, \hat{a}_i, \hat{y}_i)$ where represent the observed feature, sensitive attribute, and label, respectively, in the empirical data.

Suppose that there exists $a$ such that $\mathbb{Q}(A = a) \neq \hat{p}_a^N$. Without the loss of generality, we suppose $\mathbb{Q}(A = a) > \hat{p}_a^N$, which means
$$\frac{1}{N} \sum_{i=1}^N \mathbb{Q}_i(A = a) > \frac{1}{N} \sum_{i=1}^N \mathbbm{1}_a (\hat{a}_i).$$

This implies there exists $i^* \in [N]$ with $\hat{a}_{i^*} \neq a$ and $\mathbb{Q}_{i^*}(A = a) > 0$, such that
\begin{align*}
 W_c^2(\hat{\mathbb{P}}^N, \mathbb{Q}) &= \frac{1}{N} \sum_{i=1}^N \mathbb{E}_{\mathbb{Q}_i}[(\alpha\|X' - X\| + \infty \|A' - A\| + \beta\|Y' - Y\|)^2] \\
 & \geq \frac{1}{N} \mathbb{E}_{\mathbb{Q}_{i^*}} [(\alpha\|\hat{x}_{i^*} - X\| + \infty \|\hat{a}_{i^*} - A\| + \beta\|\hat{y}_{i^*} - Y\|)^2] \\
 & \geq \frac{1}{N} \mathbb{Q}_{i^{*}} (A = a) (\infty (\hat{a}_{i^*} - a))^2 \\
 & = \infty,
\end{align*}
which contradicts with $W_c^2(\hat{\mathbb{P}}^N, \mathbb{Q}) < \infty.$
    
\end{proof}

\begin{proof}[Proof for~\Cref{thm:marginal}]
Assume $W_c^2(\hat{\mathbb{P}}^N, \mathbb{Q}) \leq \infty$, Lemma \ref{lemma:marginal} implies $\mathcal{F}_\mathcal{R} \subseteq \mathcal{F}_\mathcal{R}(\hat{p}^N)$.

Besides $\mathcal{F}_\mathcal{R}(\hat{p}^N) \subseteq \mathcal{F}_\mathcal{R}$, since $\mathcal{F}_\mathcal{R}(\hat{p}^N)$ marginal constraints are stronger than those of $\mathcal{F}_\mathcal{R}$. 

Hence $\mathcal{F}_\mathcal{R}(\hat{p}^N) = \mathcal{F}_\mathcal{R}$, and $\mathop{inf}\limits_{\mathbb{Q} \in \mathcal{F}_\mathcal{R}} W_c^2(\hat{\mathbb{P}}^N, \mathbb{Q}) = \mathop{inf}\limits_{\mathbb{Q} \in \mathcal{F}_\mathcal{R}(\hat{p}^N)} W_c^2(\hat{\mathbb{P}}^N, \mathbb{Q}).$
\end{proof}

Before proving~\Cref{thm:dual}, we present the following lemma. 

\begin{lemma}[Interior point]
\label{lemma:slater}

Let $\Xi = X \times A \times Y$ denote the population space, and let 
\[
\hat{\Xi}^N = \{ \hat{\xi}_i = (\hat{x}_i, \hat{a}_i, \hat{y}_i) : i \in [N] \}
\]
be the empirical dataset consisting of $N$ observed samples. 
We distinguish two kinds of variables: 
\begin{itemize}
    \item $\xi = (x,a,y) \in \Xi$, a generic population element, 
    \item $\xi' \in \hat{\Xi}^N$, a variable ranging over the empirical dataset.
\end{itemize}

We define $f: \Xi \times \hat{\Xi}^N \rightarrow \mathbb{R}^{N+1}$ by
$$f(\xi, \xi') = (\mathbbm{1}_{\hat{\xi}_1}(\xi'), ..., \mathbbm{1}_{\hat{\xi}_N}(\xi'),\phi(\xi)),$$
where $$\phi(X,A,Y) = \frac{d(Y, \mathcal{R}(X)) \mathbbm{1}_1(A)}{\mathbb{Q}(A=1)} - \frac{d(Y, \mathcal{R}(X)) \mathbbm{1}_0(A)}{\mathbb{Q}(A=0)}, \text{ and}$$
$$\hat{\xi}_i = (\hat{x}_i, \hat{a}_i, \hat{y}_i).$$
Then we have 
$$\bar{q} = (\underbrace{\frac{1}{N}, ..., \frac{1}{N}}_{N}, 0)  \in \mathop{int}\{\mathbb{E}_\pi[f(\xi, \xi')]: \pi \in \mathcal{M}_+(\Xi \times \hat{\Xi}^N)\},$$
where $\mathop{int}(\cdot)$ denotes the interior (for the Euclidean topology). 
\end{lemma}
\begin{proof}
    By the definition of interior point, it suffices to find an open ball $B$ centered at $\bar{q}$
    that is completely contained in $\{\mathbb{E}_\pi[f(\xi, \xi')]: \pi \in \mathcal{M}_+(\Xi \times \hat{\Xi}^N)\}$. 
    
    Let $[0, \omega]$ denote the range of the surjective function $d$. Let $B = \left(\frac{1}{2N}, \frac{3}{2N}\right)^N \times \left(-\frac{1}{4}\omega, \frac{1}{4}\omega\right).$
    Since $B$ is an open ball centered at $\bar{q}$, it suffices to show that $\forall q \in B$, there exists $\pi \in \mathcal{M}_{+}(\Xi \times \hat{\Xi}^N)$ such that $q = \mathbb{E}_\pi[f(\xi, \xi')]$. We construct $\pi$ in the following way. $\forall a \in \mathcal{A}$, we define the locations $x_a$ and $y_a$,
    $$x_a \in \mathcal{X}, y_a \in \mathcal{Y},$$ and  $\forall i \in [N]$, set $\pi$ explicitly as, 
    $$\pi(\xi = (x_{\hat{a}_i}, \hat{a}_i, y_{\hat{a}_i}), \xi' = (\hat{x}_i, \hat{a}_i, \hat{y}_i)) = q_{i},$$ and zero elsewhere. For convenience, we write $$f(\xi, \xi') = (f_1(\xi, \xi'), ..., f_{N+1}(\xi, \xi')), $$ where for $i \leq N, f_i(\xi, \xi') = \mathbbm{1}_{\hat{\xi}_i}(\xi');$ for $i=N+1, f_{N+1}(\xi, \xi') = \phi(\xi).$
    
    By this construction, $\forall i \in [N]$,
    \begin{align*}
        \mathbb{E}_{\pi}[f_i(\xi, \xi')] = \mathbb{E}_\pi[\mathbbm{1}_{\hat{\xi}_i}(\xi')] = \sum_{(\xi,\xi')} \mathbbm{1}_{\hat{\xi}_i}(\xi') \, \pi(\xi,\xi') 
= \mathbbm{1}_{\hat{\xi}_i}(\hat{\xi}_i) \cdot q_i \;+\; 
   \sum_{\xi' \neq \hat{\xi}_i} 0 \cdot \pi(\xi,\xi') 
= q_i.
    \end{align*}

    It remains to verify $\mathbb{E}_{\pi}[f_{N+1}(\xi, \xi')] = q_{N+1}$. We define the following index set $\mathcal{I}_{a} = \{i \in [N]: \hat{a}_i = a\}$. Then
    \begin{equation}
    \begin{aligned}
        \mathbb{E}_{\pi}[f_{N+1}(\xi, \xi')] &= \mathbb{E}_\pi[\phi(\xi)] \\
        & = \mathbb{E}_\pi\left[\frac{d(Y, \mathcal{R}(X)) \mathbbm{1}_1(A)}{\hat{p}_1^N} - \frac{d(Y, \mathcal{R}(X)) \mathbbm{1}_0(A)}{\hat{p}_0^N}\right]\\
        & = (\hat{p}_1^N)^{-1}\mathbb{E}_\pi\left[d(Y, \mathcal{R}(X)) \mathbbm{1}_1(A)\right] - (\hat{p}_0^N)^{-1}\mathbb{E}\left[d(Y, \mathcal{R}(X)) \mathbbm{1}_0(A)\right]\\
        &= (\hat{p}_1^N)^{-1}d(x_1, \mathcal{R}(y_1))\sum_{i \in \mathcal{I}_{1}}q_i - (\hat{p}_0^N)^{-1}d(x_0, \mathcal{R}(y_0))\sum_{i \in \mathcal{I}_{0}}q_i. 
    \end{aligned}
    \label{eq:balance}
    \end{equation}
    It remains to find the locations of $x_1$ and $x_0$ to balance~\Cref{eq:balance} to be zero.
    Assuming $\mathbb{E}_{\pi}[f_{N+1}(\xi, \xi')] = 0$, yields
    \begin{equation}
        d(\mathcal{R}(x_1), y_1) = \frac{q_{N+1}+d(\mathcal{R}(x_0), y_0)(\hat{p}_0^N)^{-1}\sum_{i\in \mathcal{I}_{0}} q_i}{(\hat{p}_1^N)^{-1}\sum_{i \in \mathcal{I}_{1}} q_i}.
        \label{eq:balance-2}
    \end{equation}
    For individual terms in~\Cref{eq:balance-2}, we have $$(\hat{p}_0^N)^{-1}\sum_{i\in \mathcal{I}_{0}} q_i < \frac{N}{|\mathcal{I}_{0}|} \times \frac{3}{2N}  \times |\mathcal{I}_{0}| = \frac{3}{2}, \textit{and}$$
    $$(\hat{p}_1^N)^{-1}\sum_{i \in \mathcal{I}_{1}} q_i > \frac{N}{|\mathcal{I}_1|}\times \frac{1}{2N} \times |\mathcal{I}_1| = \frac{1}{2}.$$
    We have the following two cases depending on the sign of $q_{N+1}$.
    \begin{itemize}
        \item Suppose $q_{N+1} \in [0, \frac{1}{4}\omega]$. Consider picking $x_0$ and $y_0$ such that $$d(\mathcal{R}(x_0), y_0) = \frac{1}{6}\omega.$$ This choice is possible because $d$ is surjective onto $[0,\omega]$, so for any value in this interval (such as $\omega/6$) there exist inputs $(x_0, y_0)$ attaining it.
        Substituting this choice yields
        \begin{equation}
            0< d(\mathcal{R}(x_1), y_1) < \frac{\frac{1}{4}\omega+\frac{1}{6}\omega \times \frac{3}{2}}{\frac{1}{2}} = \omega.
            \label{eq:d-range-1}
        \end{equation}
        Since $d$ is a surjective function with a bounded continuous range $[0, \omega]$\footnote{In our experiments, we used absolute error or squared error for $d$. These errors are bounded for a given dataset.},~\Cref{eq:d-range-1} implies the existence of $x_1$ and $y_1$.
        \item Suppose $q_{N+1} \in [-\frac{1}{4}\omega, 0]$. Consider picking $x_0$ and $y_0$ such that $$d(\mathcal{R}(x_0), y_0) = \frac{1}{2}\omega,$$ we have 
        \begin{equation}
            0< d(\mathcal{R}(x_1), y_1) < \frac{-\frac{1}{4}\omega+\frac{1}{2}\omega \times \frac{3}{2}}{\frac{1}{2}} = \omega,
            \label{eq:d-range-2}
        \end{equation}
        which similarly implies the existence of $x_1$ and $y_1$.
    \end{itemize}
    
\end{proof}

\begin{proof}[Proof for~\Cref{thm:dual}]
We rewrite the test statistic $\mathcal{T}$ in terms of the coupling plan $\pi$,
\begin{equation}
    \mathcal{T} = \left\{
     \begin{aligned}
      \mathop{inf}_\pi \textit{ } &\mathbb{E}_\pi[c((X,A,Y), (X',A',Y'))^2] \\
      s.t. \textit{ } & \mathbb{E}_\pi[\phi(X,A,Y)] = 0\\
      & \pi(A = a) = \hat{p}_a^N, \forall a \in \mathcal{A}\\
      & \mathbb{E}_\pi[\mathbbm{1}_{(\hat{x}_1, \hat{a}_i, \hat{y}_i)}(X',A',Y')] = 1/N, \forall i \in [N]
      \end{aligned}
    \right..
    \label{eq:coupling-objective}
\end{equation}
Because of the absolute trust on sensitive attribute $a$ in~\Cref{eq:cost-function}, any coupling $\pi$ with finite Wasserstein distance should satisfy $\pi(A = a) = \hat{p}_a^N$.~\Cref{eq:coupling-objective} can be further simplified to,
\begin{equation}
    \mathcal{T} = \left\{
     \begin{aligned}
      \mathop{inf}_\pi \textit{ } &\mathbb{E}_\pi[c((X,A,Y), (X',A',Y'))^2] \\
      s.t. \textit{ } & \mathbb{E}_\pi[\phi(X,A,Y)] = 0\\
      & \mathbb{E}_\pi[\mathbbm{1}_{(\hat{x}_1, \hat{a}_i, \hat{y}_i)}(X',A',Y')] = 1/N, \forall i \in [N]
      \end{aligned}
    \right..
\end{equation}
Using the notations from Lemma~\ref{lemma:slater}, this optimization problem can be written as
\begin{equation*}
    \mathcal{T} = \mathop{inf}_\pi \{\mathbb{E}_\pi[c(\xi, \xi')^2]: \pi \in \mathcal{M}_+(\Xi \times \hat{\Xi}_N), \mathbb{E}_\pi[\phi(\xi, \xi')] = \bar{q}\}.
\end{equation*}
As Lemma~\ref{lemma:slater} has verified $\bar{q}$ is an interior point of  $\{\mathbb{E}_\pi[f(\xi, \xi')]: \pi \in \mathcal{M}_+(\Xi \times \hat{\Xi}_N)\}$, by the strong duality theorem~\citep{smith1995generalized}, 
    \begin{equation}
    \begin{aligned}
        \mathcal{T} & = \left\{
     \begin{aligned}
      \mathop{sup}\textit{ } &\frac{1}{N}\sum_{i=1}^N b_i\\
      s.t. \textit{ } & b \in \mathbb{R}^N, \gamma \in \mathbb{R}\\
      & \sum_{i=1}^N b_i \mathbbm{1}_{(\hat{x}_i, \hat{a}_i, \hat{y}_i)}(x', a', y') - \gamma \phi(x,a,y) \leq c((x,a,y), (x', a', y'))^2, \\
      & \forall (x,a,y), (x', a', y') \in \mathcal{X} \times \mathcal{A} \times {Y}
      \end{aligned}
    \right. \\
    & = \left\{
     \begin{aligned}
      \mathop{sup}\textit{ } &\frac{1}{N}\sum_{i=1}^N b_i\\
      s.t. \textit{ } & b \in \mathbb{R}^N, \gamma \in \mathbb{R}\\
      &  b_i - \gamma \phi(x,a,y) \leq c((x, a, y), (\hat{x}_i, \hat{a}_i, \hat{y}_i))^2, \\
      & \forall (x,a,y) \in \mathcal{X} \times \mathcal{A} \times {Y}, \forall i \in [N]
      \end{aligned}
    \right. \\
    & = \frac{1}{N}\mathop{sup}\limits_{\gamma}\sum_{i=1}^N \mathop{inf}
_{x \in \mathcal{X}, y \in \mathcal{Y}} \{(\alpha \|x-\hat{x}_i\|+\beta |y-\hat{y}_i|)^2 + \gamma \phi(x_i, \hat{a}_i, y_i)\}.
\end{aligned}
\label{eq:dual}
\end{equation}
\end{proof}

\begin{proof}[Proof for~\Cref{cor:dual}]
    Since the fairness criterion (equal mean) does not involve the ground-truth labels, $\beta$ in the cost function should be set to zero. Thus,~\Cref{eq:dual} can be simplified to 
    \begin{equation*}
        \mathcal{T} = \frac{1}{N}\mathop{sup}\limits_{\gamma}\sum_{i=1}^N \mathop{inf}
_{x_i \in \mathcal{X}} \{\|x_i-\hat{x}_i\|^2 + \gamma \phi(x_i, \hat{a}_i, y_i)\}.
    \end{equation*}

Recall that $\lambda(a) = (\hat{p}_1^N)^{-1}\mathbbm{1}_1(a) - (\hat{p}_0^N)^{-1}\mathbbm{1}_0(a),$ and for the equal mean criterion, the discrepancy function reduces to $d(y, \hat{y}) = \hat{y} = \mathcal{R}(x) = \rho x + \sigma$, so that
    \begin{align*}
        \phi(x,\hat{a},y) = \lambda(\hat{a})d(y, \hat{y}) = \lambda(\hat{a})\mathcal{R}(x).
    \end{align*}
Defining $\omega_i = \gamma \lambda(\hat{a}_i)$, we consider each inner optimization problem:
$$\mathop{inf}\limits_{x_i \in \mathcal{X}}\{\|x_i - \hat{x}_i\|^2 + \gamma \lambda(\hat{a}_i) R(x_i)\}.$$

For any $x_i$, we decompose it as $x_i = \hat{x}_i-k_i\omega_i\rho - k_i'\rho^{\perp}$, where $k_i, k_i' \in \mathbb{R}$, $\rho^\perp \neq 0$ and $\rho^{T}\rho^{\perp} = 0.$ Then, the objective becomes: 
    \begin{equation}
        \begin{aligned}
        \mathcal{T}_{inf} &= \mathop{inf}\limits_{x_i \in \mathcal{X}}\{\|x_i - \hat{x}_i\|^2 + \gamma \lambda(\hat{a}_i) R(x_i)\} \\
        & = \mathop{inf}\limits_{k_i, k_i' \in \mathbb{R}, \rho^{\perp} \textit{ with }{\rho^T \rho^\perp = 0}} \{\|k_i\omega_i\rho + k_i'\rho^{\perp}\|^2 + \omega_i \rho^{T}(\hat{x}_i - k_i\omega_i\rho - k'_i\rho^{\perp}) + \omega_i \sigma\} \\
        & = \mathop{inf}\limits_{k_i, k_i' \in \mathbb{R}, \rho^{\perp} \textit{ with }{\rho^T \rho^\perp = 0}} \{ \|k_i\omega_i\rho + k_i'\rho^\perp\|^2 + \omega_i \rho^T \hat{x}_i - k_i \omega^2\|\rho\|^2 + \omega_i\sigma\}. \\
        \end{aligned}
        \label{eq:omega}
    \end{equation}
    We observe~\Cref{eq:omega} achieves the infimum when $k_i'\rho^\perp = 0$ (since $\rho^\perp$ is orthogonal to $\rho$). We thus focus solely on $k_i$, and~\Cref{eq:omega} can be further simplified to 
    \begin{equation}
    \begin{aligned}
        \mathcal{T}_{inf} &= \mathop{inf}\limits_{k_i \in \mathbb{R}}\{\|\rho\|^2\omega_i^2k_i^2-\omega_i^2\|\rho\|^2k_i + \omega_i \rho^T \hat{x}_i + \omega_i\sigma\} \\
        & = -\frac{1}{4}\omega_i^2\|\rho\|^2 + \omega_i\rho^T\hat{x}_i + \omega_i\sigma,
    \end{aligned}
    \end{equation}
    where the infimum is achieved at $k_i = \frac{1}{2}.$
    Hence, we have simplified the infimum part of the saddle point problem in~\Cref{eq:dual}, then
    \begin{equation*}
        \begin{aligned}
            \mathcal{T} &= \frac{1}{N}\mathop{sup}\limits_{\gamma\in \mathbb{R}} \{\sum_{i = 1}^N - \frac{1}{4}\gamma^2\lambda(\hat{a}_i)^2\|\rho\|^2 + \gamma \lambda(\hat{a}_i)\rho^T \hat{x}_i + \gamma\lambda(\hat{a}_i)\sigma \} \\
            & = \frac{1}{N}\mathop{sup}\limits_{\gamma\in \mathbb{R}}\{ - \frac{1}{4}\|\rho\|^2(\sum_{i=1}^N \lambda(\hat{a}_i)^2) \gamma^2 + (\rho^T \sum_{i=1}^N \lambda(\hat{a}_i)\hat{x}_i + \sigma \sum_{i=1}^N \lambda(\hat{a}_i)) \gamma\} \\
            & = \frac{(\rho^T \sum_{i=1}^N \lambda(\hat{a}_i)\hat{x}_i + \sigma \sum_{i=1}^N \lambda(\hat{a}_i))^2}{N \|\rho\|^2(\sum_{i=1}^N \lambda(\hat{a}_i)^2)} \\
            & = \frac{(\sum_{i=1}^N \lambda(\hat{a}_i)(\rho^T \hat{x}_i + \sigma))^2}{N \|\rho\|^2(\sum_{i=1}^N \lambda(\hat{a}_i)^2)},
            \end{aligned}
    \end{equation*}
    as the supremum is achieved at $$\gamma^* = \frac{2\rho^T \sum_{i=1}^N \lambda(\hat{a}_i) \hat{x}_i + 2\sigma \sum_{i=1}^N \lambda(\hat{a}_i)}{\|\rho\|^2 \sum_{i=1}^N \lambda(\hat{a}_i)^2}.$$
\end{proof}

We next prove~\Cref{thm:limiting} before~\Cref{thm:asymptotic_upper_bound} because the proof for~\Cref{thm:asymptotic_upper_bound} relies on~\Cref{thm:limiting}.
\begin{proof}[Proof for~\Cref{thm:limiting}] 
We use Lemma 4 from ~\citep{blanchet2019robust} (also see~\Cref{sec:RWPI}) to prove the theorem. First, we verify the three assumptions from Lemma 4.

Since $\phi(X,A,Y)$ is bounded, $\mathbb{E}\|\phi(X,A,Y)\|^2 < \infty$, Assumption A2' (see~\Cref{sec:RWPI} in verified. 

Next, note that 
\begin{equation*}
   P\big(\|\gamma \triangledown_X d(Y, \mathcal{R}(X) \big(\frac{\mathbbm{1}_1(A)}{\hat{p}_1^N} - \frac{\mathbbm{1}_0(A)}{\hat{p}_0^N} \big)\| = 0) = P\big(\frac{\mathbbm{1}_1(A)}{\hat{p}_1^N} = \frac{\mathbbm{1}_0(A)}{\hat{p}_0^N}\big).
\end{equation*}
Since $A$ can take either values 0 and 1, and the number of elements in each demographic group is non-zero, we have $$P\big(\frac{\mathbbm{1}_1(A)}{\hat{p}_1^N} = \frac{\mathbbm{1}_0(A)}{\hat{p}_0^N}\big) = 0.$$ Hence, $$P\big(\|\gamma \triangledown_X d(Y, \mathcal{R}(X) \big(\frac{\mathbbm{1}_1(A)}{\hat{p}_1^N} - \frac{\mathbbm{1}_0(A)}{\hat{p}_0^N} \big)\| \geq 0) = 1 > 0,$$ and Assumption A4' (see~\Cref{sec:RWPI}) is verified. 

Under the local Lipschitz continuity assumption, there exists $\kappa: \mathcal{X} \times \mathcal{Y} \rightarrow [0, \infty]$ such that,  $$\frac{\|\triangledown_X \mathcal{R}((x+\triangle), y) - \triangledown_X \mathcal{R}(x, y)\|}{\|\triangle\|} \leq \kappa(x,y).$$ Therefore,
\begin{align*}
  \|(\triangledown_X \mathcal{R}((x+\triangle), y) - \triangledown_X(\mathcal{R}(x), y)) \times (\frac{\mathbbm{1}_1(A)}{\hat{p}_1^N} - \frac{\mathbbm{1}_0(A)}{\hat{p}_p^N})\|  \leq \big|\frac{\mathbbm{1}_1(A)}{\hat{p}_1^N} - \frac{\mathbbm{1}_0(A)}{\hat{p}_0^N}\big| \kappa(x, y) \|\triangle\|,
\end{align*}
so Assumption A6' (see~\Cref{sec:RWPI}) is verified.

Applying Lemma 4 from~\cite{blanchet2019robust}, we have
\begin{align*}
    N \times \mathcal{T} &\overset{d}{\rightarrow} \mathop{sup}\limits_{\gamma \in \mathbb{R}} \{\gamma \tilde{Z} - \frac{\gamma^2}{4} \mathbb{E}_\mathbb{P} [\|\triangledown_X\mathcal{R}(f(X), Y)(\frac{\mathbbm{1}_1(A)}{\hat{p}_1^N} - \frac{\mathbbm{1}_0(A)}{\hat{p}_o^N})\|^2]\} \\
    & = (\mathbb{E}_\mathbb{P}[\|\triangledown_X \mathcal{R}(f(X), Y)(\frac{\mathbbm{1}_1(A)}{\hat{p}_1^N} - \frac{\mathbbm{1}_0(A)}{\hat{p}_0^N})\|^2])^{-1}\tilde{Z}^2,
\end{align*}
where $\tilde{Z} \sim N(0, Cov(\phi(X,A,Y))).$

We now study the behavior of $\phi(X,A,Y))$ under the null hypothesis that $\mathbb{P} \in \mathcal{F}_\mathcal{R}$. Specifically, $$\frac{1}{p_0} \mathbb{E}_\mathbb{P}[d(Y, \mathcal{R}(X))\mathbbm{1}_0(A)] = \frac{1}{p_1} \mathbb{E}_\mathbb{P}[d(Y, \mathcal{R}(X))\mathbbm{1}_1(A)].$$

Let $H^N = \frac{1}{N}\sum_{i=1}^N \phi(\hat{x}_i, \hat{a}_i, \hat{y}_i)$ be a consistent estimator for $\phi(X,A,Y)$.
\begin{align*}
    H^N &= \frac{1}{N}\sum_{i=1}^N \phi(\hat{x}_i, \hat{a}_i, \hat{y}_i) \\
    &= \frac{1}{N} \sum_{i=1}^N d(y_i, \mathcal{R}(\hat{x}_i)) (\frac{\mathbbm{1}_1(\hat{a}_i)}{\hat{p}_1^N} - \frac{\mathbbm{1}_0(\hat{a}_i)}{\hat{p}_0^N})\\
    &=\frac{1}{\hat{p}_0^N \hat{p}_1^N} \frac{1}{N} \sum_{i=1}^N d(y_i, \mathcal{R}(\hat{x}_i)) (\hat{p}_0^N\mathbbm{1}_1(\hat{a}_i) - \hat{p}_1^N\mathbbm{1}_0(\hat{a}_i)) \\
    & = \frac{1}{\hat{p}_0^N \hat{p}_1^N} \big\{\frac{1}{N} \sum_{i=1}^N d(y_i, \mathcal{R}(\hat{x}_i)) (p_0\mathbbm{1}_1(\hat{a}_i) - p_1\mathbbm{1}_0(\hat{a}_i)) \\
    & +  \frac{1}{N} (\hat{p}_0^N-p_0) \sum_{i=1}^N d(y_i, \mathcal{R}(\hat{x}_i))\mathbbm{1}_1(\hat{a}_i) -  \frac{1}{N} (\hat{p}_1^N-p_1) \sum_{i=1}^N d(y_i, \mathcal{R}(\hat{x}_i)) \mathbbm{1}_0(\hat{a}_i)\big\}. \\
\end{align*}
By Slutsky's theorem, we have
\begin{align*}
    &\frac{1}{N} (\hat{p}_0^N-p_0) \sum_{i=1}^N d(y_i, \mathcal{R}(\hat{x}_i)) \mathbbm{1}_1(\hat{a}_i)\overset{d.}{\rightarrow} (\hat{p}_0^N-p_0) \mathbb{E}_\mathbb{P}[d(y_i, \mathcal{R}(\hat{x}_i)) \mathbbm{1}_1(\hat{a}_i)], \\
    &\frac{1}{N} (\hat{p}_1^N-p_1) \sum_{i=1}^N d(y_i, \mathcal{R}(\hat{x}_i)) \mathbbm{1}_0(\hat{a}_i)\overset{d.}{\rightarrow} (\hat{p}_1^N-p_1) \mathbb{E}_\mathbb{P}[d(y_i, \mathcal{R}(\hat{x}_i)) \mathbbm{1}_0(\hat{a}_i)].
\end{align*}
Under the null hypothesis $\mathbb{P} \in \mathcal{F}_\mathcal{R}$, we have
\begin{align*}
    H^N & = \frac{1}{\hat{p}_0^N \hat{p}_1^N} \big\{\frac{1}{N} \sum_{i=1}^N d(y_i, \mathcal{R}(\hat{x}_i)) (p_0\mathbbm{1}_1(\hat{a}_i) - p_1\mathbbm{1}_0(\hat{a}_i)) \\
    & + (\hat{p}_0^N-p_0) \mathbb{E}_\mathbb{P}[d(y_i, \mathcal{R}(\hat{x}_i)) \mathbbm{1}_1(\hat{a}_i)] - (\hat{p}_1^N-p_1) \mathbb{E}_\mathbb{P}[d(y_i, \mathcal{R}(\hat{x}_i)) \mathbbm{1}_0(\hat{a}_i)]\big\} \\
    & = \frac{1}{\hat{p}_0^N \hat{p}_1^N} \big\{\frac{1}{N} \sum_{i=1}^N d(y_i, \mathcal{R}(\hat{x}_i)) (p_0\mathbbm{1}_1(\hat{a}_i) - p_1\mathbbm{1}_0(\hat{a}_i)) \\
    & + \hat{p}_0^N \mathbb{E}_\mathbb{P}[d(\hat{y}_i, \mathcal{R}(\hat{x}_i)) \mathbbm{1}_1(\hat{a}_i)] - \hat{p}_1^N \mathbb{E}_\mathbb{P}[d(y_i, \mathcal{R}(\hat{x}_i)) \mathbbm{1}_0(\hat{a}_i)]\big\} \\
    & = \frac{1}{\hat{p}_0^N \hat{p}_1^N} \big\{\frac{1}{N} \sum_{i=1}^N d(y_i, \mathcal{R}(\hat{x}_i)) (p_0\mathbbm{1}_1(\hat{a}_i) - p_1\mathbbm{1}_0(\hat{a}_i)) \\
    & + \frac{1}{N} \sum_{i=1}^N \mathbbm{1}_0(\hat{a}_i) \mathbb{E}_\mathbb{P}[d(\hat{y}_i, \mathcal{R}(\hat{x}_i)) \mathbbm{1}_1(\hat{a}_i)] - \frac{1}{N} \sum_{i=1}^N \mathbbm{1}_1(\hat{a}_i) \mathbb{E}_\mathbb{P}[d(y_i, \mathcal{R}(\hat{x}_i)) \mathbbm{1}_0(\hat{a}_i)]\big\} \\
    & = \frac{1}{\hat{p}_0^N \hat{p}_1^N} \big\{\frac{1}{N} \sum_{i=1}^N d(y_i, \mathcal{R}(\hat{x}_i)) (p_0\mathbbm{1}_1(\hat{a}_i) - p_1\mathbbm{1}_0(\hat{a}_i)) \\
    & + \mathbbm{1}_0(\hat{a}_i) \mathbb{E}_\mathbb{P}[d(\hat{y}_i, \mathcal{R}(\hat{x}_i)) \mathbbm{1}_1(\hat{a}_i)] - \mathbbm{1}_1(\hat{a}_i) \mathbb{E}_\mathbb{P}[d(\hat{y}_i, \mathcal{R}(\hat{x}_i)) \mathbbm{1}_0(\hat{a}_i)] \big\}.\\
    & \overset{d.}{\rightarrow} \tilde{Z}
\end{align*}
By central limit theorem,
$\tilde{Z} \sim N(0, Cov(Z'))$ where 
\begin{align*}
   Z' & = \frac{1}{p_0 p_1} \{d(Y, \mathcal{R}(X))(p_0\mathbbm{1}_1(A) - p_1\mathbbm{1}_0(A)) \\
   & + \mathbbm{1}_0(A) \mathbb{E}_\mathbb{P}[d(Y, \mathcal{R}(X)) \mathbbm{1}_1(A)] - \mathbbm{1}_1(A) \mathbb{E}_\mathbb{P}[d(Y, \mathcal{R}(X)) \mathbbm{1}_0(A)]\}. 
\end{align*}
\end{proof}

\begin{proof}[Proof for~\Cref{thm:asymptotic_upper_bound}] 
In this proof, to differentiate between the test statistics from exact expectation equivalence and expectation equivalence within a tolerance, we use $\mathcal{T}$ to denote the test statistic from exact expectation equivalence, and $\mathcal{T}^{tol}$ to denote the one from expectation equivalence within a tolerance. 

In~\Cref{thm:limiting},  we have shown that 
$$N \times \mathcal{T} \overset{d}{\rightarrow} \theta \chi_1^2,$$ where $\mathcal{T} = \mathop{inf}\limits_{\mathbb{Q} \in \mathcal{F}_\mathcal{R}} W_c^2(\mathbb{P}, \mathbb{Q}),$ and $\mathcal{F}_\mathcal{R}$ is the set of distributions satisfying exact expectation equivalence. 

By definition of the limiting distribution, this implies $$\overline{\mathop{lim}}_{n \rightarrow \infty} \mathbb{E}[f(N \times \mathcal{T})] \leq \mathbb{E}[f(\theta \chi_1^2)]$$ for every continuous and bounded non-decreasing function $f$.

Let $\mathcal{F}_{\mathcal{R}^{tol}}$ be the set of distributions satisfying expectation equivalence within tolerance. By definition, we have $\mathcal{F}_{\mathcal{R}^{tol}} \subseteq \mathcal{F}_\mathcal{R}.$ Then,  $$\mathcal{T}^{tol} = \mathop{inf}\limits_{\mathbb{Q} \in \mathcal{F}_\mathcal{R'}} W_c^2(\mathbb{P}, \mathbb{Q}) \leq \mathcal{T}.$$

Hence, $$\overline{\mathop{lim}}_{n \rightarrow \infty} \mathbb{E}[f(N \times \mathcal{T'})] \leq \overline{\mathop{lim}}_{n \rightarrow \infty} \mathbb{E}[f(N \times \mathcal{T})] \leq \mathbb{E}[f(\theta \chi_1^2)],$$

which  matches the definition of the asymptotic upper bound presented in the theorem.
\end{proof}

\begin{proof}[Proof for~\Cref{thm:optimal-perturbation}]
\label{proof:optimal-perturbation}
Let $\sum_{i=1}^N\triangledown(R(x_i)))$ be the original fairness metric value from all the data, and $\sum_{i=1}^N\Delta(R(x_i^\eta)$ be the perturbed value from all the data.

For any optimal solution $(x^\eta,y^\eta)$ from the inner optimization problem, the KKT conditions require:
\[\forall i, \phi(x_i^\eta,\hat{a},y_i^\eta) + \eta\delta_i = 0,\] which implies
\begin{align*}
 \forall i, \phi(x_i^\eta,\hat{a},y_i^\eta) &= -\eta\delta_i,   \\
 \sum_{i=1}^N \phi(x_i^\eta,\hat{a},y_i^\eta) &= -\eta \sum_{i=1}^N \delta_i.  
\end{align*}

Thus, the perturbed data improves the fairness violation by a multiplicative factor of $\eta$, as claimed.
\end{proof}

\begin{proof}[Proof for~\Cref{cor:optimal-perturbation-special}]
From the proof for~\Cref{cor:dual}, we know that the optimal $x_i^*$ for the inner optimiazation problem is $$x_i^* = \hat{x}_i - \frac{1}{2}\gamma^*\rho^T \lambda(\hat{a}_i).$$

Since $x_i^*$ are the data perturbation with perfectly fair predictions, and by the linearlity of linear regression and equal mean criterion in the special case, the Wasserstein interporlation becomes linear too. Hence, if we want to reduce the unfairness by a degree of $\eta$, $$x_i^\eta = \hat{x}_i - \frac{1}{2}\eta \gamma^*\rho^T \lambda(\hat{a}_i).$$

Since there is no label involved in the equal mean criterion, $y_i^\eta$ can be set to any value, for example, $y_i^\eta = \hat{y}_i$.
\end{proof}

\section{Validation of the limiting distribution}
\label{sec:validation-limiting-distribution}
We consider a regression setting where
\begin{equation}
\begin{aligned}
    X &\sim \mathcal{N}(0,1), \\
    \hat{Y} &= 3X+1, \\
    p_0 &= 0.6, p_1 = 0.4.
\end{aligned}
\label{eq:sim2}
\end{equation}
In~\Cref{fig:empirical-limiting}, we visualize the limiting distribution computed using~\Cref{cor:special-case-limiting-distribution} and the Wasserstein projection distance computed using~\Cref{cor:dual} for $N=1000$ in \textbf{A} and $N=10000$ in \textbf{B}. We observe that as $N$ increases, the empirical distribution of $N\times \mathcal{T}$ converges to $\theta \chi_1^2$, which confirms the correctness of our limiting distribution. 

\begin{figure}[!h]
    \centering
    \includegraphics[width=0.9\linewidth]{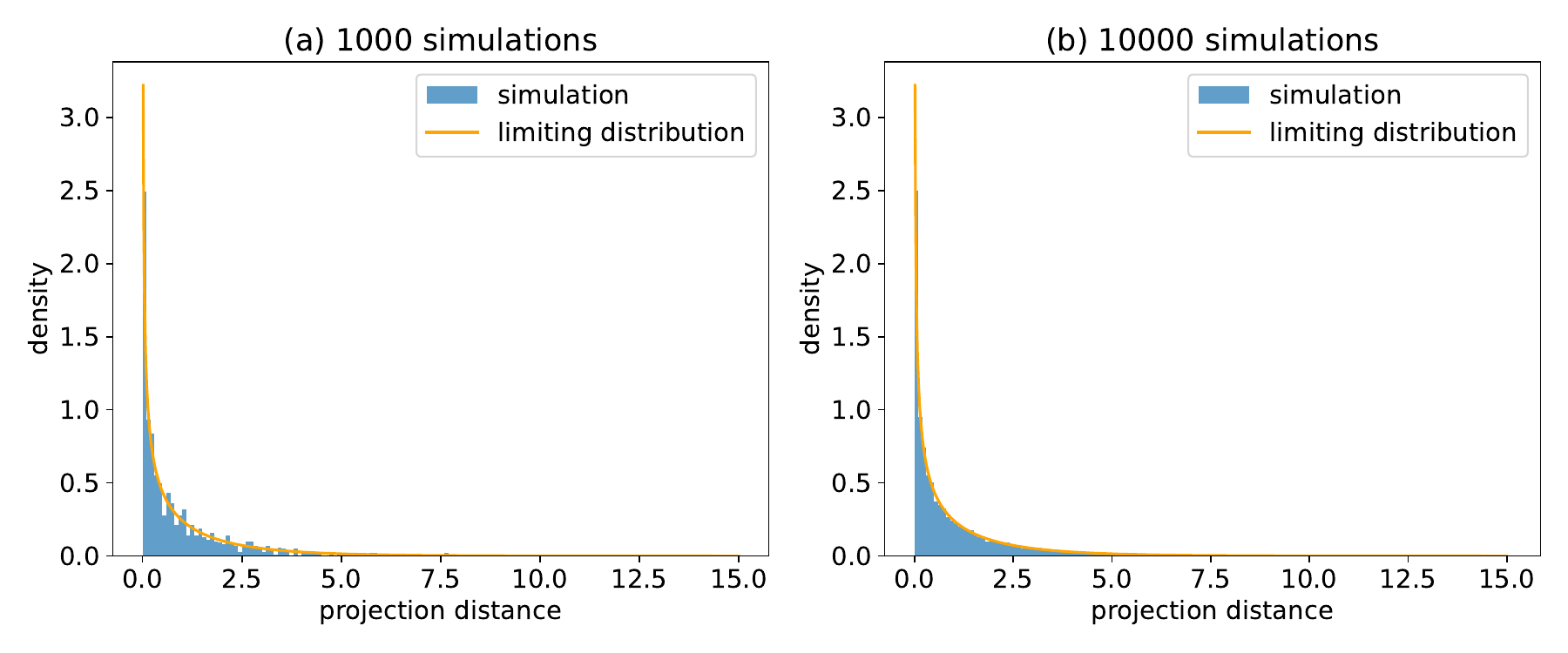}
    \caption{Empirical distribution of the Wasserstein projection distance versus limiting distribution for (a) 1000 simulations and (b) 10000 simulations. We observe that as the number of simulations increases, the empirical distribution converges to the limiting distribution.}
    \label{fig:empirical-limiting}
\end{figure}

\section{Validation of the optimal data perturbation theorem}
\label{sec:optimal-perturbation-special-case}
In this section, we validate the theoretical result from~\Cref{cor:optimal-perturbation-special}, which describes how to correct a linear regression model to improve fairness under the equal mean fairness criterion. Specifically, we simulate a dataset where the sensitive attribute $A \in \{0,1\}$ influences both the features $X$ and the outcome $Y$, and then apply the perturbation procedure described in~\Cref{cor:optimal-perturbation-special}.

We evaluate model fairness for equal mean fairness criterion using the difference in means of the corrected predictions, and model accuracy using the relative mean absolute error between the corrected predictions and the true targets. The perturbation strength is controlled by a parameter $\eta \in [0,1]$, where larger values induce stronger data perturbations and more fairness at a potential cost to accuracy.

\Cref{fig:optimal-perturbation-special-case} illustrates the trade-off between fairness and accuracy as $\eta$ increases. As expected, higher values of $\eta$ lead to reduced group mean differences, indicating improved fairness, while slightly increasing the relative MAE, representing a modest decrease in prediction accuracy. This behavior aligns with our theoretical result in~\Cref{cor:optimal-perturbation-special}, confirming that the data perturbation improves fairness while keeping the corrected empirical distribution close to the original empirical distribution.

\begin{figure}
    \centering
    \includegraphics[width=0.6\linewidth]{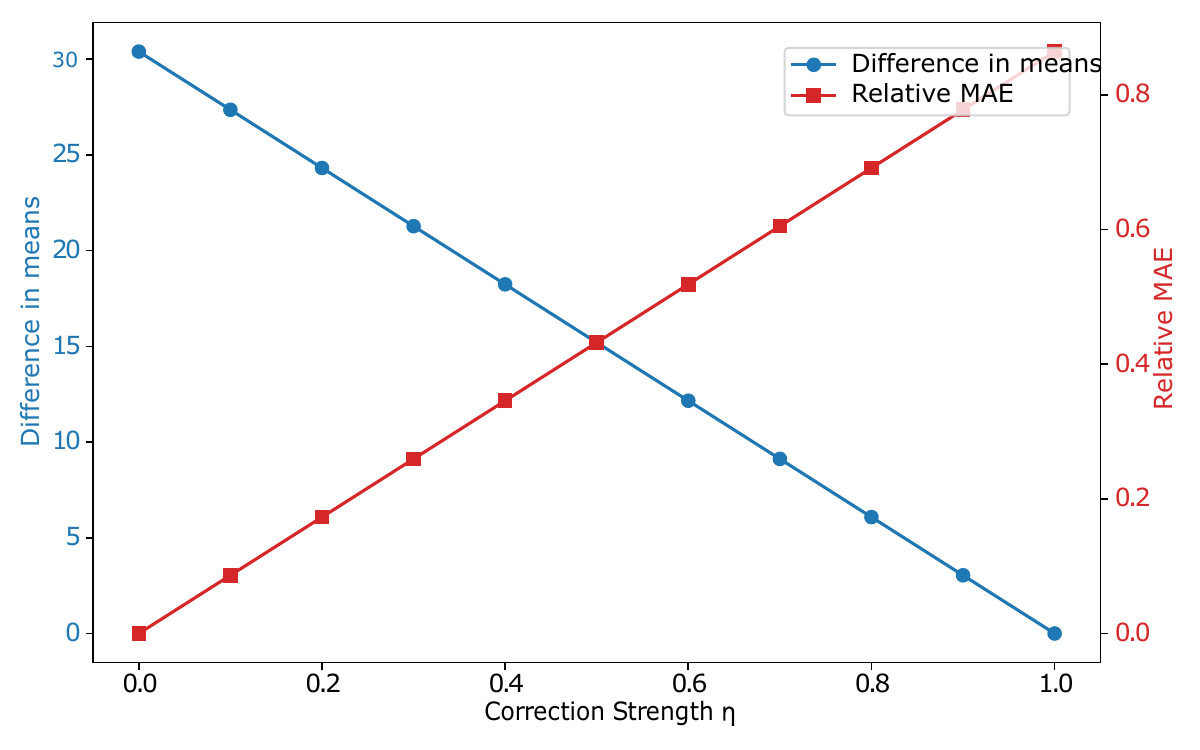}
    \caption{Scatterplot of difference in means and relative MAE scores versus correction strength $\eta$.}
    \label{fig:optimal-perturbation-special-case}
\end{figure}

\section{Simulation details}
\label{sec:simulation-details}

We conducted simulations to compare the Wasserstein projection-based test with the permutation test under the equal mean fairness criterion in a linear regression setting. Datasets were generated with two groups defined by a binary sensitive attribute $A \in \{0,1\}$. For each configuration, we varied one factor while fixing the others and repeated the experiment 100 times to estimate power and specificity.  

\paragraph{Sample size.} Total sample sizes ranged from 40 to 160, equally split between the two groups. Larger sample sizes reduce variance in test statistics, thereby improving power.  

\paragraph{Effect size.} Group differences were introduced by shifting the conditional means of outcomes in different groups, controlled by a parameter in $[0,1]$. Greater shifts correspond to stronger violations of fairness and yield higher power.  

\paragraph{Significance level.} Tests were evaluated at significance thresholds from 0.01 to 0.20. Higher thresholds increase the chance of rejecting the null, raising power but lowering specificity.  

Power was computed as the proportion of trials where the null was correctly rejected under unfair conditions, while specificity was computed as the proportion of trials where the null was correctly retained under fair conditions. These definitions align with conventional statistical testing and ensure comparability across methods.

\section{Other fairness criteria for regression}
\label{sec:other-criteria}
\paragraph{\textbf{Average ratio}} The ratio of independence, separation and sufficiency~\citep{steinberg2020fairness} can be expressed as,
$$r_{ind} = \frac{\mathop{Pr}(\mathcal{R}(X)|A = 1)}{\mathop{Pr}(\mathcal{R}(X)|A = 0)}, r_{sep} = \frac{\mathop{Pr}(\mathcal{R}(X)|A = 1, Y)}{\mathop{Pr}(\mathcal{R}(X)|A = 0, Y)}, r_{suf} = \frac{\mathop{Pr}(Y|A = 1, \mathcal{R}(X))}{\mathop{Pr}(Y|A = 0, \mathcal{R}(X))}. $$

Perfect independence, separation and sufficiency would yield a constant ratio of 1 for all $X$. Pragmatically it would be more useful to know these ratio in expectation known as the \emph{average ratio}. Hence, enforcing $$\mathop{\mathbb{E}}\limits_{X}[r_{ind}] = 1, \mathop{\mathbb{E}}\limits_{X}[r_{sep}] = 1 \textit{ and } \mathop{\mathbb{E}}\limits_{X}[r_{suf}] = 1$$ provides another relaxed way to enforce the fairness criteria.

\section{Robust Wasserstein profile inference and its limit theorem}
\label{sec:RWPI}
\emph{Robust Wasserstein Profile Inference}~\citep{blanchet2019robust} is a methodology which extends the use of methods inspired by empirical likelihood to the setting of optimal transport costs (of which Wasserstein distances are a particular case). This paper derives general limit theorems for the asymptotic distribution of the \emph{Robust Wasserstein Profile} (RWP) function defined for general estimating equations. We set up our notations to introduce one limit theorem presented in the paper.

Suppose $X$ and $Y$ are random variables, $h$ is an integrable function, $\mathbb{P}$ is the data generating distribution for $X$ and $Y$, $\mathbb{P}_n$ is the empirical distribution of $X$ and $Y$, and $W_c$ is a Wasserstein distance with cost function being $c(w, u) = \|w - u\|_q^\rho$.  Assuming the RWP function for estimating $\theta_*$ satisfies $\mathbb{E}[h(W, \theta_*)] = 0,$ we define the RWP function as,
$$R_n(\theta_*; \rho) = \mathop{inf}\{W_c(\mathbb{P}, \mathbb{P}_n): \mathbb{E}_\mathbb{P} [h(X, Y; \theta_*)] = 0\}.$$

Lemma 4 shows that under the following assumptions:
\begin{itemize}
    \item A2': Suppose that $\theta_* \in \mathbb{R}^d$ satisfies 
\[
\mathbb{E}[h(X,Y;\theta_\ast)] = 0
\quad \text{and} \quad 
\mathbb{E}\|h(X,Y;\theta_\ast)\|_2^2 < \infty.
\]
While we do not assume that $\theta_\ast$ is unique, the results are stated for a fixed $\theta_\ast$ satisfying $\mathbb{E}[h(X,Y;\theta_\ast)] = 0$;
\item A4': Suppose that for each $\xi \neq 0$, the partial derivative $D_x h(x,y;\theta_\ast)$ satisfies
\[
\mathbb{P}\!\left(\, \|\xi^\top D_x h(X,Y;\theta_\ast)\|_p > 0 \,\right) > 0;
\]
\item A6': Assume that there exists $\bar{\kappa}:\mathbb{R}^m \to [0,\infty)$ such that
\[
\| D_x h(x+\Delta, y;\theta_\ast) - D_x h(x,y;\theta_\ast) \|_p 
\;\leq\; \bar{\kappa}(x,y)\,\|\Delta\|_q,
\quad \forall \Delta \in \mathbb{R}^d,
\]
and $\mathbb{E}[\bar{\kappa}(X,Y)^2] < \infty$.
\end{itemize}

Then we have, for $\rho \geq 2$, 
$$nR_n(\theta_*; \rho) \overset{d}{\rightarrow} \bar{R}(\rho),$$ where
$$\bar{R}(\rho) = \mathop{sup}\limits_{\xi}\{\rho \xi^T H - (\rho-1)\mathbb{E}\|\xi^T D_x h(X, Y; \theta_*)\|_p^{\rho/(\rho-1)}\},$$
with $H \sim \mathcal{N}(\mathbf{0}, \mathop{Cov}[h(X, Y;\theta_*)])$ and $1/p + 1/q = 1.$

\end{document}